\documentclass{article}

\usepackage[preprint]{mstyle}
\PassOptionsToPackage{table}{xcolor}

\usepackage[utf8]{inputenc}
\usepackage[T1]{fontenc}
\usepackage{microtype}
\usepackage{url}
\usepackage{xspace}
\usepackage[normalem]{ulem}

\usepackage{amsmath,amssymb,amsfonts,amsthm}
\usepackage{mathtools}
\usepackage{bm}

\usepackage{graphicx}
\usepackage{caption}
\usepackage{subcaption}
\usepackage{xcolor}
\usepackage{colortbl}
\usepackage{booktabs}
\usepackage{multirow}
\usepackage{threeparttable}
\usepackage{adjustbox}
\usepackage{arydshln}
\usepackage{nicefrac}
\usepackage{pifont}
\usepackage{algorithm}
\usepackage{algpseudocode}
\usepackage{scalefnt}

\definecolor{linkcolor}{RGB}{83,83,182}
\usepackage{hyperref}
\hypersetup{
  colorlinks=true,
  citecolor=linkcolor,
  linkcolor=linkcolor,
  urlcolor=linkcolor
}
\usepackage[nameinlink,capitalize]{cleveref}

\crefname{section}{Sec.}{Secs.}
\Crefname{section}{Sec.}{Secs.}
\crefname{table}{Tab.}{Tabs.}
\Crefname{table}{Tab.}{Tabs.}
\crefname{figure}{Fig.}{Figs.}
\Crefname{figure}{Fig.}{Figs.}
\crefname{algorithm}{Alg.}{Algs.}
\Crefname{algorithm}{Alg.}{Algs.}
\crefname{theorem}{Thm.}{Thms.}
\Crefname{theorem}{Thm.}{Thms.}
\crefname{lemma}{Lem.}{Lems.}
\Crefname{lemma}{Lem.}{Lems.}
\crefname{proposition}{Prop.}{Props.}
\Crefname{proposition}{Prop.}{Props.}

\definecolor{mydarkgreen}{rgb}{0,0.42,0}
\definecolor{mydarkred}{rgb}{0.75,0,0}
\definecolor{mygreen1}{rgb}{0,0.8,0}
\definecolor{mygreen2}{RGB}{0,120,20}
\definecolor{RoyalBlue}{RGB}{65,105,225}
\definecolor{Gray}{gray}{0.95}
\definecolor{bgcolor}{rgb}{0.93,0.99,1}
\definecolor{bgcolor2}{rgb}{0.8,1,0.8}
\definecolor{bgcolor3}{rgb}{0.50,0.90,0.50}
\definecolor{bgcolor4}{rgb}{0.94,0.97,1}

\newcommand{\gbox}{\colorbox{bgcolor4}}
\newcommand{\gboxeq}[1]{\text{\colorbox{bgcolor4}{$#1$}}}

\DeclareRobustCommand{\algname}[1]{%
  \ifmmode
    \text{\textup{\ttfamily #1}}%
  \else
    \textup{\ttfamily #1}%
  \fi
  \xspace
}
\newcommand{\algn}[1]{{\sf\scalefont{0.90}{#1}}\xspace}
\newcommand{\algnn}[1]{{\sf\scalefont{0.90}{#1}}}

\newcommand{\eqdef}{\coloneqq}

\newcommand{\EE}{\mathbb{E}}
\newcommand{\E}{\mathbb{E}}

\newcommand{\R}{\mathbb{R}}

\newcommand{\F}{\mathcal{F}}

\newcommand{\ste}{\algn{STE}}
\newcommand{\method}{\algn{JacQuant}}
\newcommand{\jvr}{\algn{JacQuant-VR}}
\newcommand{\jbase}{\algn{JacQuant-Base}}
\newcommand{\probe}{\algn{JacQuant-Probe}}
\newcommand{\dither}{\algn{JacQuant-Dither}}

\newcommand{\GradEst}{\texttt{GradEst}\xspace}
\newcommand{\RefGrad}{\texttt{RefGrad}\xspace}
\newcommand{\ProbeUpdate}{\texttt{ProbeUpdate}\xspace}
\newcommand{\CtrlUpdate}{\texttt{CtrlUpdate}\xspace}
\newcommand{\DitherEMA}{\texttt{DitherEMA}\xspace}

\theoremstyle{plain}
\newtheorem{theorem}{Theorem}
\newtheorem{proposition}{Proposition}
\newtheorem{lemma}{Lemma}
\newtheorem{corollary}{Corollary}

\theoremstyle{definition}

\newtheorem{assumption}{Assumption}

\theoremstyle{remark}

\crefname{section}{Sec.}{Secs.}
\Crefname{section}{Sec.}{Secs.}
\crefname{algorithm}{Alg.}{Algs.}
\Crefname{algorithm}{Alg.}{Algs.}
\crefname{table}{Tab.}{Tabs.}
\Crefname{table}{Tab.}{Tabs.}
\crefname{figure}{Fig.}{Figs.}
\Crefname{figure}{Fig.}{Figs.}
\crefname{equation}{Eq.}{Eqs.}
\Crefname{equation}{Eq.}{Eqs.}
\crefname{theorem}{Thm.}{Thms.}
\Crefname{theorem}{Thm.}{Thms.}
\crefname{lemma}{Lem.}{Lems.}
\Crefname{lemma}{Lem.}{Lems.}
\crefname{corollary}{Cor.}{Cors.}
\Crefname{corollary}{Cor.}{Cors.}
\crefname{proposition}{Prop.}{Props.}
\Crefname{proposition}{Prop.}{Props.}
\crefname{definition}{Def.}{Defs.}
\Crefname{definition}{Def.}{Defs.}
\crefname{assumption}{Assump.}{Assumps.}
\Crefname{assumption}{Assump.}{Assumps.}
\crefname{remark}{Rem.}{Rems.}
\Crefname{remark}{Rem.}{Rems.}
\crefname{example}{Ex.}{Exs.}
\Crefname{example}{Ex.}{Exs.}
\crefname{appendix}{App.}{Apps.}
\Crefname{appendix}{App.}{Apps.}

\newcommand{\rev}[1]{{\color{blue}#1}}

\usepackage{tabularx}
\usepackage{subcaption}
\usepackage{array}

\newcolumntype{Y}{>{\centering\arraybackslash}X}

\newcommand{\panelcaption}[1]{%
  \noindent{\footnotesize\bfseries #1}\par\vspace{2pt}%
}

\usepackage{enumitem}

\title{JacQuant: STE-Free Quantization-Aware Training\\
via Learned Jacobian Surrogates}

\author{%
  Kai Yi\thanks{Correspondence to: \texttt{kaiyi96@meta.com}.}\\
  Meta AI\\
  \And
  Vignesh Vivekraja\\
  Meta AI\\
  \And
  Harshit Khaitan\\
  Meta AI\\
  \And 
  Steven Li\\
  Meta AI
}

\begin{document}

\maketitle

\begin{abstract}
Quantization-aware training (QAT) is widely deployed but typically relies on the Straight-Through Estimator (\ste), which passes gradients through non-differentiable quantizers by fiat. This often makes training brittle near bin boundaries and weakly aligned with the actual behavior of the low‑precision model. 
We introduce \method, a QAT framework that learns a lightweight surrogate of the model’s local sensitivity to parameter changes and uses it to stabilize and accelerate training within standard variance‑reduced optimizers. The surrogate is inexpensive (diagonal or block‑diagonal), data‑driven, and compatible with common weight and activation quantizers.
On “code‑preserving” training phases, we prove convergence for non‑convex objectives and obtain linear rates under a PŁ condition, and we relate the learned sensitivity to end‑to‑end output fidelity via a simple calibration argument. 
Across LLM benchmarks at $\leq2$ bits, \method consistently reaches higher accuracy than \algnn{STE}-based QAT, and the runtime analyses on various models show that the added cost remains negligible under practical group sizes. The method is drop-in and requires no changes to the forward quantizers; our empirical claims are scoped to ultra-low-bit LLM QAT.
\end{abstract}
\section{Introduction}

Quantization is a primary route to deploy large language models (LLMs) under tight memory and latency budgets by reducing parameter precision and improving hardware utilization \citep{aqlm, gptq, tseng2024quip, liu2024vptq, liu2025paretoq}. In many practical deployments, \emph{post-training quantization} (PTQ) remains attractive because it avoids gradient-based training and can be applied quickly to a pretrained checkpoint \citep{gptq, tseng2024quip}. However, when pushing to \emph{ultra-low precision}, most notably $\leq2$-bit weights, PTQ often suffers sharp degradation, and \emph{quantization-aware training} (QAT) becomes important to recover accuracy by adapting scales, clipping, and representations to the task distribution \citep{liu2025paretoq, winq2025}.

Almost all modern QAT relies on the \emph{Straight-Through Estimator} (\ste) \citep{bengio2013estimating}: the forward pass uses a hard quantizer $Q_\Delta(\cdot)$, but backpropagation pretends its Jacobian is the identity. This creates a persistent mismatch because hard quantizers are piecewise constant and dominated by bin boundaries and saturation/clipping at low bit-widths \citep{yin2019understanding,esser2020learned,nagel2020up}.
In the $\leq 2$-bit regime, where many parameters lie near boundaries or hit clipping limits, \algnn{STE}’s identity Jacobian can over-propagate gradients, leading to instability and suboptimal solutions.

\textbf{Key idea}: \emph{learn the missing quantizer Jacobian.} We view QAT as a \emph{missing-Jacobian} problem and introduce \method (\algnn{Jac}obian learning for \algnn{Quant}ization), which keeps the \emph{deployed hard quantizer} in the forward pass but replaces \algnn{STE}’s fixed identity Jacobian with a learned surrogate sensitivity map $B(W)$. Let $q=Q_\Delta(W)$ and $v=\nabla_q\ell$. \method uses $\nabla_W \ell \approx B(W)\,v,$ where $B(W)$ is diagonal or block-diagonal with one scalar per quantization group, making it cheap and easy to integrate. We train $B(W)$ to track the quantizer’s \emph{local sensitivity}: near $1$ in bin interiors and near $0$ under active saturation, producing geometry-aware gradients that naturally damp saturated directions and reduce oscillations near boundaries.

In \Cref{fig:tissue}, we contrast our approach with representative prior methods. \ste treats quantization as gradient-transparent; fixed-transform methods (e.g., Rotation Trick \citep{rotationtrick}) alter gradient geometry in a static way. In contrast, \method learns an \emph{adaptive} backward gain $B(W)$ that aligns gradient propagation with how the hard quantizer actually responds to perturbations, especially under clipping/saturation where \ste is most misaligned.

\begin{figure}[!tb] 
    \centering
    \includegraphics[width=0.9\linewidth, trim = 250 400 950 500, clip]{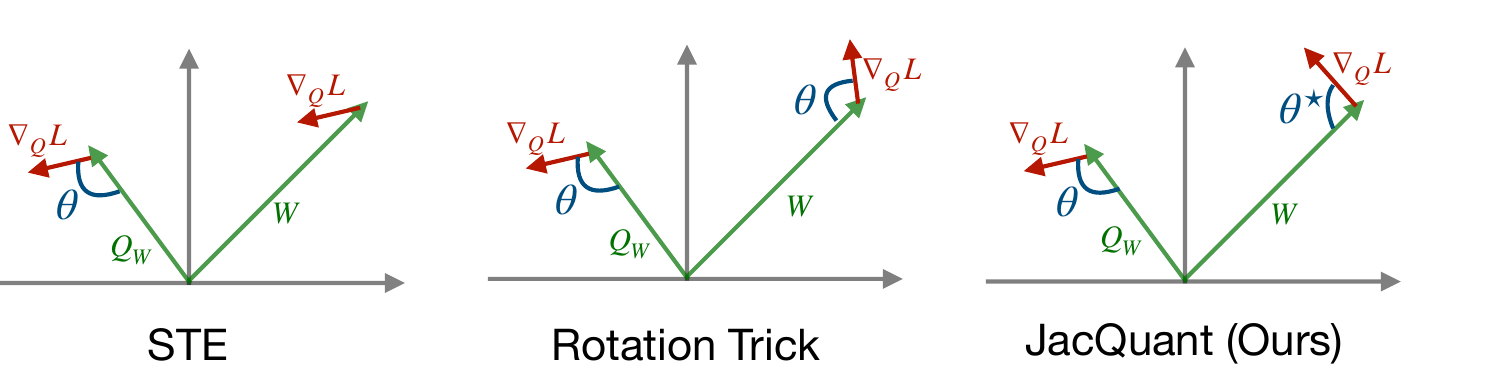}
    \vspace{-2mm}
    \caption{Gradient propagation schemes in quantized optimization.
    \emph{Left}: \ste, which directly copies the forward gradient direction.
    \emph{Middle}: \algn{Rotation Trick}~\citep{rotationtrick}, which applies a fixed relative rotation $\theta$ to the gradient.
    \emph{Right}: \method (ours), which learns an adaptive surrogate Jacobian $B(W)$, yielding an optimal rotation $\theta^{\star}$ together with an optimal gradient rescaling.}
    \label{fig:tissue}
    \vspace{-2mm}
\end{figure}

To keep the additional cost negligible, \method updates the surrogate Jacobian $B(W)$ only occasionally and amortizes it across many QAT steps. Sensitivity is estimated using either random \emph{probes} or \emph{subtractive dithering}, requiring only a single scalar per group and preserving standard QAT throughput. While the framework is compatible with classical variance-reduction or control-variate techniques, the central departure from \ste lies in explicitly learning the Jacobian rather than assuming an identity surrogate.

Our \textbf{contributions} are summarized as follows:
\begin{itemize}[leftmargin=1.2em]
    \item \textit{STE-free backward rule via Jacobian learning:} we replace STE’s identity Jacobian with a learned, group-wise sensitivity map $B(W)$ while keeping the hard forward quantizer unchanged.
    \item \emph{Practical estimators with negligible overhead.} We propose two efficient estimators (\algn{Probe} and \algn{Dither}) that update $B(W)$ intermittently and scale naturally with group-wise quantization. 
    \item \emph{Theory for a principled target gradient.} Using a smoothed/dithered quantization model, we relate \algnn{JacQuant}’s updates to a well-defined target gradient and establish standard non-convex convergence guarantees (and linear rates under a PŁ condition) when $B(W)$ tracks the mean sensitivity \citep{karimi2016linear}. 
    \item \emph{Extensive empirical validation on LLMs.} We demonstrate that \method is a drop-in upgrade to \ste-based QAT pipelines and improves stability and accuracy in $\leq 2$-bit regimes on LLM benchmarks. Empirically, we evaluate \method on \texttt{LLaMA3-1B/3B}, \texttt{Qwen3-1.7B}, and a \texttt{LLaMA3-8B} W2A16 trajectory; compare against \algn{Rotation Trick} under matched \algn{ParetoQ} settings; and report runtime overhead on both \texttt{LLaMA3-1B} and \texttt{LLaMA3-8B}. We scope the empirical claim to ultra-low-bit LLM QAT while isolating the algorithmic contribution to the backward rule.
\end{itemize}
\section{Related Work}
\vspace{-1mm}
\paragraph{Post-training quantization (PTQ).}
PTQ calibrates a pretrained model without gradient updates and remains a strong default for rapid LLM deployment. Representative methods include \algn{SmoothQuant}~\citep{xiao2023smoothquant} (activation outlier smoothing for W8A8), \algn{GPTQ}~\citep{gptq} (Hessian-informed weight-only quantization), \algn{ZeroQuant}/\algn{ZeroQuant-V2}~\citep{yao2022zeroquant,yao2023zeroquant} (per-layer PTQ with hardware-friendly kernels), and \algn{AWQ}~\citep{lin2024awq} (activation-aware protection of salient channels for 4-bit inference).

\paragraph{Quantization-aware training (QAT).}
QAT typically improves accuracy in aggressive regimes (e.g., $\leq$2-bit weights and/or W\&A quantization) by adapting scales
and clipping to the task distribution rather than relying purely on calibration statistics.
Classic \ste-based QAT spans early work such as \algn{DoReFa-Net}~\citep{zhou2016dorefa}, \algn{PACT}~\citep{choi2018pact},
and step-size learning methods such as \algn{LSQ}/\algn{LSQ+}~\citep{esser2020learned,bhalgat2020lsq+}.
Recent findings similarly suggest QAT is often preferable to PTQ in ultra-low-bit settings
(e.g., \algn{ParetoQ}~\citep{liu2025paretoq}, \algn{WinQ}~\citep{winq2025}).

\paragraph{Beyond the Straight-Through Estimator.}
Most QAT methods propagate gradients through quantizers using the Straight-Through Estimator (\ste), which can be biased near
bin boundaries and poorly aligned with the true low-precision objective. Two notable departures are \algn{PV-Tuning}~\citep{malinovskii2024pv}, which uses a proximal surrogate, and the \algn{Rotation Trick}~\citep{rotationtrick}, which applies a fixed geometric alignment.
Other complementary routes include differentiable relaxations/annealing (e.g., \algn{Soft-to-Hard VQ}~\citep{agustsson2017soft})
and hybrid first/zeroth-order corrections (e.g., \algn{FOGZO}~\citep{yang2025fogzo}), as well as optimizer-level advances (e.g., \algn{Muon}~\citep{liu2025muon}). Zeroth-order approaches such as \algn{ZeroQAT}~\citep{zeroqat} address a complementary cost regime: they change the gradient-estimation mechanism, whereas \method isolates the backward rule through the deployed hard quantizer and amortizes probing cost across many QAT steps.
We pursue \emph{Jacobian learning} because it preserves the deployed hard quantizer in the forward pass, yet replaces the fixed \ste with a lightweight learned sensitivity map $B(W)$ that (i) \emph{reduces bias} by matching local quantization sensitivity,
(ii) \emph{stabilizes} training via adaptive damping near saturation/clipping, and (iii) amortizes any probing cost by updating
$B(W)$ only occasionally. An expanded comparison is provided in \Cref{app:beyond_ste}.

\paragraph{Variance reduction.}
Control-variate methods (e.g., \algn{SAG}~\citep{roux2012stochastic}, \algn{SVRG}~\citep{johnson2013accelerating}, \algn{SAGA}~\citep{defazio2014saga}) and Jacobian-sketching (\algn{JacSketch}~\citep{gower2021stochastic}) reduce stochastic noise in large-scale optimization. We adapt these ideas to stabilize Jacobian-modulated QAT updates with minimal overhead; see \Cref{app:vr} and \ref{app:vr-api}.
\section{Method}\label{sec:method}
\vspace{-1mm}
We introduce \method, a QAT framework that replaces \ste with a \emph{learned surrogate Jacobian}. Rather than treating quantization as ``identity in backward", we learn how perturbations of full-precision weights \rev{affect} their quantized proxies and use this sensitivity to map gradients back to weight space. 

\subsection{Preliminaries: Grouped Quantization in LLMs}\label{sec:prelim}
\vspace{-1mm}
Consider a linear layer with weights $W\in\mathbb{R}^{d_{\text{out}}\times d_{\text{in}}}$ and input $x$, producing $z=Wx$.
QAT uses a quantized proxy $\widehat{W}=Q_\Delta(W)$ in the forward pass.
We use \emph{grouped quantization}: partition weights into groups $\{W^{(g)}\}_{g=1}^G$ (e.g., 128 weights per group) and quantize each group independently.
For a sample $(x_i,y_i)$, let
$v_i := \nabla_{\widehat{W}}\,\ell(f(x_i;\widehat{W}),y_i)\big|_{\widehat{W}=Q_\Delta(W)}$.
The challenge is that $Q_\Delta(\cdot)$ is piecewise constant (often with clipping), so its true Jacobian is missing.

\subsection{The Core Idea: Learning the Quantization Jacobian}\label{sec:core}
We optimize the standard supervised objective
\begin{equation}\label{eq:obj}
L(W)=\mathbb{E}_{(x,y)\sim\mathcal D}\big[\ell(f(x;Q_\Delta(W)),y)\big].
\end{equation}
\ste corresponds to using an identity surrogate for the missing Jacobian.
\method instead learns a lightweight sensitivity map $B(W)$ and uses the backward rule
\begin{equation}\label{eq:jacquant-core}
\gboxeq{
\nabla_W L(W) \approx B(W)\,\nabla_{\widehat{W}}L(\widehat{W})\big|_{\widehat{W}=Q_\Delta(W)} .
}
\end{equation}

\paragraph{Parameterization and cost.}
We use a block-diagonal $B(W)$ with one scalar per quantization group, $B(W)=\mathrm{blkdiag}(b_1 I,\dots,b_G I)$. With group size 128 this adds one scalar per 128 weights (negligible vs.\ optimizer state).

\subsection{Practical Estimation of $B(W)$}\label{sec:algorithms}
\method updates $B(W)$ only on refresh steps (Alg.~\ref{alg:jacquant-unified}) and reuses it across many QAT iterations.

\textbf{\probe (default).}
For each group $g$, draw $\delta_g\sim\mathcal N(0,\sigma^2 I)$ and measure $\Delta q_g = Q_\Delta(W_g+\delta_g)-Q_\Delta(W_g)$. A one-step slope fit gives
\begin{equation}\label{eq:probe-update}
\hat b_g=\frac{\langle \Delta q_g,\delta_g\rangle}{\|\delta_g\|_2^2+\epsilon},
\quad
b_g \leftarrow (1-\beta)b_g+\beta\,\mathrm{clip}(\hat b_g,0,1).
\end{equation}

\textbf{\dither.}
We also consider subtractive dithering $r_g\sim \mathcal U[-\Delta/2,\Delta/2]$ and the de-dithered proxy $q_g=Q_\Delta(W_g+r_g)-r_g$, and update $b_g$ toward the empirical sensitivity of $W_g\mapsto q_g$. We provide the full estimator and implementation details in \Cref{app:vr,app:vr-api}.

\begin{algorithm}[t]
\caption{\jvr\ (unified view)}
\label{alg:jacquant-unified}
\small
\begin{algorithmic}[1]
\Require $W_0$, dataset $\mathcal D$, quantizer $Q_\Delta$, stepsize $\eta$, refresh prob $p\in(0,1]$, mode $\in\{\textsc{Probe},\textsc{Dither}\}$, EMA rate $\beta$
\State $B_0\leftarrow I$; \quad anchors $(\tilde W,\tilde B)\leftarrow(W_0,B_0)$ \Comment{{start from the STE gain}}
\State VR state $s_0$; \quad reference gradient $\tilde g \leftarrow \textsc{RefGrad}(\tilde W,\tilde B)$ \Comment{{anchor/control variate}}
\For{$t=0,1,\dots$}
  \State Sample minibatch $S_t\subset\mathcal D$
  \State $g_t \leftarrow \textsc{GradEst}(W_t,B_t,\tilde W,\tilde B,\tilde g,s_t;S_t)$ \Comment{{backprop uses learned $B_t$}}
  \State $W_{t+1} \leftarrow W_t - \eta\, g_t$ \Comment{{standard QAT update}}
  \State $s_{t+1} \leftarrow \textsc{CtrlUpdate}(s_t,W_{t+1},\tilde W;S_t)$ \Comment{{update VR memory if used}}
  \If{$u\sim\mathrm{Unif}[0,1]\le p$ \textbf{ (or criterion $\mathcal C$ holds)}} \Comment{{refresh rarely}}
      \If{\textsc{Probe}}
        \State $B_{t+1}\leftarrow(1{-}\beta)B_t{+}\beta\,\textsc{ProbeUpdate}(W_{t+1})$ \Comment{{fit local slope}}
      \Else
        \State $B_{t+1}\leftarrow(1{-}\beta)B_t{+}\beta\,\textsc{DitherUpdate}(W_{t+1})$ \Comment{{dithered sensitivity}}
      \EndIf
      \State $(\tilde W,\tilde B)\leftarrow(W_{t+1},B_{t+1})$ \Comment{{synchronize anchor}}
      \State $\tilde g \leftarrow \textsc{RefGrad}(\tilde W,\tilde B)$ \Comment{{refresh reference}}
  \Else
      \State $B_{t+1}\leftarrow B_t$ \Comment{{reuse amortized Jacobian}}
  \EndIf
\EndFor
\end{algorithmic}
\end{algorithm}

\subsection{Optional Variance Reduction}\label{sec:vr}
Jacobian learning adds stochasticity (from \algnn{probes}/\algnn{dithers}), so \method can be combined with scalable variance-reduced estimators via occasional refresh steps (Alg.~\ref{alg:jacquant-unified}). \Cref{app:vr} gives the concrete control-variate forms. Disabling the control-variate state yields \jbase (Alg.~\ref{alg:jacquant-unified-no-vr}), isolating the effect of Jacobian learning.
\providecommand{\GradEst}{\mathrm{GradEst}}
\providecommand{\RefGrad}{\mathrm{RefGrad}}

\section{Convergence Analysis of \method}
\label{sec:convergence}
\paragraph{Analytical storyline.}
The analysis is organized around one question: \textit{what should replace the identity Jacobian used by \ste, and what happens to optimization when we use this replacement?}  The argument has four linked parts.  First, because a hard quantizer is piecewise constant, we analyze short \emph{code-preserving windows} and introduce a smoothed reference gradient that has a well-defined quantizer sensitivity $J(W)$.  Second, we show that the only extra bias in \method is the mismatch $B(W)-J(W)$, whereas \ste has the fixed mismatch $I-J(W)$.  Third, we prove that the probe and dither updates estimate this same $J(W)$ and therefore make the learned backward map competitive with, and often less biased than, \ste.  Finally, once this surrogate gradient is bounded and locally regular, standard stochastic/variance-reduced convergence rates apply with two interpretable residuals: a Jacobian-tracking bias and a stochastic-variance floor.
Thus the proof should be read as a chain rather than a list of independent claims: \Cref{lem:bias-to-target} defines the target that \ste misses, \Cref{lem:probe-consistency,prop:sa-tracking,lem:dither-fixedpoint} explain how the algorithm learns that target, and \Cref{lem:stability,lem:vr-bias,thm:nonconvex-generic,thm:pl-generic} show that optimization behaves like a standard smooth stochastic method once the remaining Jacobian error and variance are accounted for.

\subsection{Local Smoothness and the Target Gradient}
\label{subsec:setup}
\label{subsec:target-gradient}
Quantization $Q_\Delta:\R^d\to\R^d$ has derivative zero almost everywhere and undefined derivative on cell boundaries.  Directly analyzing $\nabla_W L(Q_\Delta(W))$ is therefore uninformative.  We instead examine short intervals during training in which the quantization code does not change.  This assumption is not an algorithmic restriction; it is a local analytical device matching the practical fact that QAT weights usually cross quantization cells only intermittently.

\begin{assumption}[Code preservation]
\label{as:code}
There exist $t_0$ and $T\ge1$ such that $Q_\Delta(W_t)=Q_\Delta(W_{t_0})\equiv q$ for all $t\in[t_0,t_0+T)$.
\end{assumption}

Within such a window, let $v_i:=\nabla_q\ell(f(x_i;q),y_i)$ and $\bar v:=n^{-1}\sum_i v_i$.  To define a meaningful gradient through the hard quantizer, we smooth the quantizer by subtractive dithering:
$$
    m(W):=\E_r[Q_\Delta(W+r)-r],\qquad r\sim\mathrm{Unif}[-\Delta/2,\Delta/2],
$$
and define the mean-field quantizer sensitivity $J(W):=\nabla_W m(W)$.  For unclipped bin interiors, $J(W)$ is close to the identity; near clipping or saturation, its diagonal entries shrink toward zero.  Thus the principled target gradient is
$$
    g^\dagger(W)=J(W)\bar v,
$$
which is the gradient of the dither-averaged quantized objective on the window.  \method uses $g_B(W)=B(W)\bar v$, whereas \ste uses $g^{\ste}(W)=\bar v$.

\begin{lemma}[Bias to the target gradient]
\label{lem:bias-to-target}
On a code-preserving window,
$$
\|g_B(W)-g^\dagger(W)\|\le \|B(W)-J(W)\|\,\|\bar v\|,
\qquad
\|g^{\ste}(W)-g^\dagger(W)\|\le \|I-J(W)\|\,\|\bar v\|.
$$
\end{lemma}

\noindent\textit{Why this lemma is the core comparison.}
The lemma isolates the backward-rule question from all other parts of QAT.  If a coordinate is in a stable bin interior, $J(W)\approx I$ and \ste is already reasonable; \method is designed to recover the same behavior because $B(W)$ is initialized/updated toward identity-like sensitivity there.  The difference appears near clipping and saturation, where $J(W)\preceq I$ and \ste still transmits the full upstream gradient.  In that regime, learning $B(W)\approx J(W)$ produces the desired damping and directly reduces the target-gradient bias.

\subsection{Why the Learned $B(W)$ Tracks $J(W)$}
\label{subsec:B-tracking}
The previous lemma is useful only if $B(W)$ can actually approach $J(W)$.  The next results explain this in the same order as the algorithm: probing estimates a local slope, the EMA recursion tracks the slowly moving target, and dithering gives an alternative estimator with the same population fixed point.
This connects directly to the refresh block of \Cref{alg:jacquant-unified}: each refresh obtains a noisy measurement of local quantizer response, while the EMA prevents that measurement noise from becoming a high-frequency change in the backward rule.

\begin{lemma}[Probe-LS consistency on a fixed window]
\label{lem:probe-consistency}
Fix a code-preserving window and a group $g$ of size $d_g$.  Draw $m$ i.i.d. probes $\delta_k\sim\mathcal N(0,\sigma^2I)$ and observe $\Delta q_k=Q_\Delta(W+\delta_k)-Q_\Delta(W)$.  Let $\widehat B_g$ minimize $\sum_{k=1}^m\|\Delta q_{g,k}-B_g\delta_{g,k}\|^2$ over the diagonal or block-diagonal class used by \method.  Then, for a universal constant $C$ and any $\delta\in(0,1)$,
$$
\Pr\!\left(\|\widehat B_g-J_g(W)\|\le C\sigma^{-1}\sqrt{\frac{d_g+\log(1/\delta)}{m}}\right)\ge1-\delta.
$$
\end{lemma}

\noindent\textit{What the probe lemma says.}
The probe update is not an arbitrary learned gate: it is a local regression of quantized output changes $\Delta q$ on weight perturbations $\delta$.  On a fixed window, the cross-moment relation $\E[\Delta q\,\delta^\top]=J(W)\E[\delta\delta^\top]$ makes the least-squares target equal to the mean-field sensitivity.  The bound also explains the group-size trade-off seen in ablations: finer groups can represent more heterogeneous sensitivity, but they require more reliable slope estimates; overly noisy estimates inject unnecessary backward-rule variance.

\begin{proposition}[EMA/SA tracking with drifting $W$]
\label{prop:sa-tracking}
Under Robbins--Monro stepsizes, persistent probe excitation, compact projection of $B_t$, and slower drift of $W_t$, the EMA recursion for $B_t$ satisfies
$$
\E[\|B_{t+1}-J(W_t)\|\mid\mathcal F_t]
\le \rho\|B_t-J(W_t)\|+\zeta_t,
$$
for some $0<\rho<1$ and $\zeta_t\to0$.  Hence $B_t$ tracks $J(W_t)$ in probability, and almost surely when the drift errors are summable.
\end{proposition}

\noindent\textit{Why tracking matters.}
The static probe estimate would be insufficient if the target sensitivity changed too quickly.  The proposition states the needed separation of time scales: QAT updates move the weights, but $B_t$ is refreshed often enough, and with enough excitation, to follow the moving mean-field sensitivity.  This justifies amortizing the Jacobian updates over many training steps instead of re-estimating them every iteration.

\begin{lemma}[Dither fixed point equals $J(W)$]
\label{lem:dither-fixedpoint}
With uniform subtractive dithering, $J(W)=\nabla_W\E_r[Q_\Delta(W+r)-r]$ is the unique population fixed point of the dither-induced update on a code-preserving window.
\end{lemma}

\noindent\textit{Role of the dither lemma.}
Probe and dither use different measurements, but they target the same object.  Dithering randomizes thresholds so that the expected hard quantizer becomes differentiable; its derivative is exactly the desired mean-field sensitivity.  Therefore the two practical estimators are not separate heuristics, but two ways of learning the same missing Jacobian.

\begin{corollary}[Data-driven dominance over \ste]
\label{cor:dominance}
Let $\gamma(W_t):=\|I-J(W_t)\|$.  If, after tracking, $\|B_t-J(W_t)\|\le \gamma(W_t)-\epsilon$, then
$$
\|g_{B_t}(W_t)-g^\dagger(W_t)\|
\le
\|g^{\ste}(W_t)-g^\dagger(W_t)\|-\epsilon\|\bar v_t\|.
$$
\end{corollary}

\noindent\textit{Interpretation.}
The corollary turns the tracking results into the practical claim: once $B_t$ is closer to $J(W_t)$ than the identity is, \method has a strictly smaller bias to the smoothed target gradient.  In bin interiors, $\gamma(W_t)$ is small and \method should behave similarly to \ste; near saturation, $\gamma(W_t)$ is large and the learned damping can provide a meaningful advantage.
This is the point where the analysis links back to the observed training behavior: \method does not need to change the hard forward quantizer; it only changes how much upstream gradient is trusted in groups where the forward quantizer is locally insensitive.

\subsection{Stability of the Surrogate Gradient}
\label{subsec:properties}
Bias reduction alone is not enough; the surrogate direction must also be stable enough for stochastic optimization.  We impose a mild regularity condition that matches the implemented diagonal/block-diagonal parameterization with clipping of group gains.

\begin{assumption}[Boundedness and regularity of $B$]
\label{as:B}
$B(W)$ is diagonal or block-diagonal, uniformly bounded as $\|B(W)\|\le\beta_B$, and locally Lipschitz on each code-preserving window.
\end{assumption}

\begin{lemma}[Spectral stability]
\label{lem:stability}
Under \Cref{as:B}, $\|B(W)v_i\|\le\beta_B\|v_i\|$ for every per-example backpropagation signal $v_i$.  When the learned sensitivity tracks clipping/saturation with gain below one, it damps directions that \ste would transmit unchanged.
\end{lemma}

\noindent\textit{Why this is not just a boundedness lemma.}
The learned Jacobian acts as an automatic gain controller.  Its gain remains close to one where the quantizer is locally responsive, but drops in regions where the quantized forward map is insensitive.  This is precisely the ``saturation damping'' effect: the backward pass no longer pushes hard in directions where the forward quantizer cannot respond proportionally.

\begin{lemma}[Local Lipschitzness]
\label{lem:lipschitz}
On a fixed code assignment, $F_i(W)=B(W)v_i$ is locally Lipschitz, and the corresponding windowed surrogate dynamics satisfy the smoothness conditions required by the descent analysis.
\end{lemma}

\noindent\textit{How this connects to convergence.}
\Cref{lem:stability,lem:lipschitz} convert the discontinuous training problem into a locally smooth surrogate problem for the purpose of analyzing updates.  The discontinuity is still present in the forward pass, but the backward directions used inside a window are bounded, regular, and aligned with the mean-field sensitivity.

\subsection{Windowed and End-to-End Rates}
\label{subsec:vr-framework}
\label{subsec:windowed}
Let the stochastic/variance-reduced estimator be
\begin{equation}
\label{eq:vr-estimator}
g_t=\GradEst(W_t,B_t,\tilde W,\tilde B,\tilde g,s_t;S_t),
\end{equation}
where $(\tilde W,\tilde B,\tilde g)$ is an occasionally refreshed anchor state.  The proof uses the following standard ABC-style oracle bound, which covers the loopless-SVRG/control-variate form used in \method while keeping the theorem independent of a particular implementation.
At this stage the discontinuous-quantizer issue has been reduced to an imperfect-oracle issue: the estimator may be biased by $B_t-J(W_t)$ and noisy because of minibatches/anchors, and the convergence rates track exactly these two quantities.

\begin{assumption}[\jvr with general bounds]
\label{as:vr}
There is a nonnegative state-variance proxy $\sigma_t$ and constants $A,B,C,\tilde A,\tilde B,\tilde C\ge0$ with $\tilde B<1$ such that, on a code-preserving window,
\begin{align}
\EE[\|g_t-\nabla\widetilde L(W_t)\|^2\mid\mathcal F_t]
&\le 2A(\widetilde L(W_t)-\widetilde L^\star)+B\sigma_t+C, \tag{ABC-1}\label{eq:abc-1-main}\\
\EE[\sigma_{t+1}\mid\mathcal F_t]
&\le 2\tilde A(\widetilde L(W_t)-\widetilde L^\star)+\tilde B\sigma_t+\tilde C. \tag{ABC-2}\label{eq:abc-2-main}
\end{align}
The anchor/reference gradient has bounded staleness and contributes an effective variance constant $\rho\sigma_{\rm vr}^2$ with $\rho\in(0,1]$.
\end{assumption}

\begin{lemma}[Bias/variance decomposition under VR]
\label{lem:vr-bias}
Under \Cref{as:vr}, the estimator error decomposes as
$$
\EE\|g_t-\nabla\widetilde L(W_t)\|^2
\le O(\|B_t-J(W_t)\|^2\|\bar v_t\|^2)+O(\rho\sigma_{\rm vr}^2).
$$
\end{lemma}

\noindent\textit{Meaning of the decomposition.}
The first term is the Jacobian-learning error already identified in \Cref{lem:bias-to-target}; it vanishes as $B_t$ tracks $J(W_t)$.  The second term is the ordinary stochastic-optimization variance after control-variate/anchor reduction.  Thus VR tightens constants, while the learned Jacobian corrects the bias caused by the hard quantizer.

\begin{theorem}[Non-convex windowed convergence]
\label{thm:nonconvex-generic}
Under \Cref{as:code,as:B,as:vr} and $L_s$-smoothness on the window, for $\eta\le c/L_s$,
$$
\min_{0\le t<T}\EE\|\nabla\widetilde L(W_t)\|^2
\le O(1/(\eta T))+O(\varepsilon^2)+O(\rho\sigma_{\rm vr}^2\eta),
\quad
\varepsilon:=\sup_t\EE\|B_t-J(W_t)\|\,\|\bar v_t\|.
$$
\end{theorem}

\noindent\textit{Interpretation.}
The rate has the classical $O(1/T)$ stationarity term plus two transparent residuals.  The $O(\varepsilon^2)$ term measures how much the learned backward map still deviates from the mean-field quantizer sensitivity; the $O(\rho\sigma_{\rm vr}^2\eta)$ term is the stochastic variance floor.  Accurate probe/dither tracking reduces the former, and anchor/control-variate updates reduce the latter.

\begin{assumption}[P\L\ inequality]
\label{as:pl}
There exists $\mu>0$ such that $\frac12\|\nabla\widetilde L(W)\|^2\ge\mu(\widetilde L(W)-\widetilde L^\star)$ on the window.
\end{assumption}

\begin{theorem}[P\L\ regime]
\label{thm:pl-generic}
Under \Cref{as:code,as:B,as:vr,as:pl} and $\eta\le1/(2L_s)$,
$$
\EE[\widetilde L(W_t)-\widetilde L^\star]
\le (1-\eta\mu)^t(\widetilde L(W_0)-\widetilde L^\star)
+O(\rho\sigma_{\rm vr}^2\eta/\mu+\varepsilon^2/\mu).
$$
\end{theorem}

\noindent\textit{Interpretation.}
In the P\L\ regime, \method inherits the same linear contraction as smooth stochastic optimization.  The asymptotic neighborhood is smaller when $B_t$ is closer to $J(W_t)$ and when the VR anchor lowers the variance.  Spectral damping further helps in practice by suppressing large updates in saturated regions before they can create oscillations.

\subsection{Composition Across Windows and Learner Convergence}
\label{subsec:global}
\label{subsec:B-convergence}
The windowed analysis remains meaningful across training because code changes can be treated as a sequence of small perturbations to the local surrogate, provided the changes become milder as training stabilizes.
Equivalently, a full QAT trajectory stitches together many local smooth problems; the drift term below measures the cost of moving from one local problem to the next.

\begin{assumption}[Vanishing inter-window drift]
\label{as:drift}
For successive code-preserving windows $k$, the corresponding surrogates satisfy $\|\nabla\widetilde L^{(k+1)}(W)-\nabla\widetilde L^{(k)}(W)\|\le\delta_k$ on visited iterates, with $\delta_k\downarrow0$ and $\sum_k\delta_k<\infty$.
\end{assumption}

\begin{theorem}[Global composition across windows]
\label{thm:global}
If each window satisfies \Cref{thm:nonconvex-generic} or \Cref{thm:pl-generic}, then the same rates hold with an additional drift term $O(\delta_k^2)$ in the non-convex bound and $O(\delta_k/\mu)$ in the P\L\ bound.  If code assignments eventually stabilize, the P\L\ trajectory is globally linearly convergent; if small flips persist with $\delta_k\to0$, the convergence neighborhood shrinks accordingly.
\end{theorem}

\begin{assumption}[SA conditions for $B$]
\label{as:sa}
The $B$ updates use Robbins--Monro stepsizes, persistent probe/dither excitation, compact projection for the chosen diagonal/block-diagonal class, a locally contractive mean-field operator, and slower drift of $W_t$.
\end{assumption}

\begin{lemma}[Mean-field identification]
\label{lem:meanfield-id}
On a fixed window, both the probe and dither mean-field operators have $J_g(W)$ as the unique fixed point within the chosen group structure.
\end{lemma}

\begin{proposition}[Convergence and stability of the Jacobian learner]
\label{prop:B-fixedpoint}
Under \Cref{as:code,as:sa} and \Cref{lem:meanfield-id}, $B_t^{(g)}\to J_g(W)$ almost surely on a fixed window, and under slow drift it tracks $J_g(W_t)$ with the contraction form in \Cref{prop:sa-tracking}.
\end{proposition}

\noindent\textit{Overall conclusion.}
The lemmas and theorems serve distinct roles.  \Cref{lem:bias-to-target} identifies the bias that \method must reduce; \Cref{lem:probe-consistency,prop:sa-tracking,lem:dither-fixedpoint,cor:dominance} show that the learned map can reduce it; \Cref{lem:stability,lem:lipschitz} guarantee stable surrogate gradients; and \Cref{thm:nonconvex-generic,thm:pl-generic,thm:global,prop:B-fixedpoint} show that these ingredients preserve standard convergence behavior while improving the constants most affected by low-bit saturation.

\section{Experiments}
We present a comprehensive empirical evaluation of \method on modern LLMs under ultra-low-bit QAT. Our experiments are designed to (i) validate the central thesis that \emph{learning a quantization Jacobian surrogate} improves optimization over the fixed \ste surrogate, (ii) quantify end-task benefits on both perplexity and downstream reasoning accuracy, and (iii) demonstrate that the gains come at negligible memory/runtime cost. Throughout, \method is a \emph{drop-in backward-rule replacement}: the forward quantizer, model, data, and optimizer remain unchanged.

\paragraph{Experimental setup.}\label{sec:exp_setup_main}
We evaluate \method on three representative LLM families (\texttt{LLaMA3-1B/3B} \citep{llama3} and \texttt{Qwen3-1.7B} \citep{yang2025qwen3}) under symmetric uniform QAT with per-channel scaling, focusing on ultra-low-bit regimes (W$\le$2, A8/A16). We also include a \texttt{LLaMA3-8B} W2A16 trajectory and a matched \algn{Rotation Trick} \citep{rotationtrick} comparison on \texttt{LLaMA3-1B}. We compare against popular QAT methods \algn{ParetoQ} \citep{liu2025paretoq} and \algn{WinQ} \citep{winq2025}. All methods share identical model architectures, training data, optimization settings, and forward quantization, and differ only in their backward gradient rules. Scale/clipping handling is kept exactly the same as in the corresponding base QAT pipeline; \method adds only the learned backward sensitivity map $B(W)$ and does not introduce a separate scale/clipping advantage. We report validation perplexity on WikiText-2 and zero-shot accuracy on eight standard benchmarks. Appendix~\ref{app:exp-setup} provides the detailed experimental settings and baseline definitions.

\begin{table}[tb]
\centering
\caption{QAT results under symmetric uniform quantization.}
\label{tab:qat_results_all}

\scriptsize
\setlength{\tabcolsep}{2.6pt}
\renewcommand{\arraystretch}{1.10}

\begin{minipage}[t]{0.505\textwidth}
\vspace{2pt}
\phantomsubcaption\label{tab:main_results}
\panelcaption{(a) \texttt{LLaMA3-1B}, A16.}

\begin{tabularx}{\linewidth}{@{}l*{6}{Y}@{}}
\toprule
& \multicolumn{2}{c}{W1A16} 
& \multicolumn{2}{c}{W1.58A16} 
& \multicolumn{2}{c}{W2A16}\\
\cmidrule(lr){2-3} 
\cmidrule(lr){4-5} 
\cmidrule(lr){6-7}
Metrics 
& PPL & Acc. 
& PPL & Acc. 
& PPL & Acc.\\
\midrule
FP16 Baseline 
& 9.6 & 58.5 
& 9.6 & 58.5 
& 9.6 & 58.5 \\
\midrule
\algn{RTN} (\algn{PTQ}) 
& 4.2e8 & 33.7 
& 1.8e6 & 36.2 
& 1.5e6 & 38.5\\
\algn{GPTQ} (\algn{PTQ}) 
& 3.3e8 & 32.7 
& 4.6e4 & 32.8 
& 3.3e2 & 36.8 \\
\algn{SpinQuant} (\algn{PTQ}) 
& 2.4e8 & 33.7 
& 2.2e3 & 32.6 
& 46.7 & 38.3 \\
\midrule
\algn{ParetoQ} 
& 17.9 & 51.9 
& 14.8 & 54.3 
& 13.5 & 55.7 \\
\ \ $+$Rotation Trick & 18.1 & 52.0 & 14.2 & 54.7 & 13.3 & 56.1\\
\ \ $+$\method 
& \textbf{17.6} & \textbf{52.4} 
& \textbf{13.9} & \textbf{55.3} 
& \textbf{12.3} & \textbf{56.6}\\
\midrule
\algn{WinQ} 
& 15.3 & 52.6 
& 12.9 & 55.6 
& 11.9 & 56.6 \\
\ \ $+$\method 
& \textbf{15.1} & \textbf{53.5} 
& \textbf{12.3} & \textbf{56.1} 
& \textbf{11.8} & \textbf{56.8}\\
\bottomrule
\end{tabularx}
\end{minipage}
\hfill
\renewcommand{\arraystretch}{0.9}
\begin{minipage}[t]{0.475\textwidth}
\vspace{2pt}

\phantomsubcaption\label{tab:complimentary_results}
\panelcaption{(b) \texttt{LLaMA3-3B} and \texttt{Qwen3-1.7B}, A8.}

\begin{tabularx}{\linewidth}{@{}l*{8}{Y}@{}}
\toprule
& \multicolumn{4}{c}{\texttt{LLaMA3-3B}} 
& \multicolumn{4}{c}{\texttt{Qwen3-1.7B}} \\
\cmidrule(lr){2-5} 
\cmidrule(lr){6-9}
& \multicolumn{2}{c}{W1A8} 
& \multicolumn{2}{c}{W1.58A8} 
& \multicolumn{2}{c}{W1A8} 
& \multicolumn{2}{c}{W2A8}\\
\cmidrule(lr){2-3} 
\cmidrule(lr){4-5} 
\cmidrule(lr){6-7} 
\cmidrule(lr){8-9}
Metrics 
& PPL & Acc. 
& PPL & Acc. 
& PPL & Acc. 
& PPL & Acc. \\
\midrule
FP16 Base 
& 7.7 & 65.2 
& 7.7 & 65.2 
& 16.2 & 58.1 
& 16.2 & 58.1 \\
\midrule
\algn{ParetoQ} 
& 15.7 & 54.0 
& 13.1 & 55.9 
& 46.5 & 42.2 
& 22.2 & 47.8 \\
\ \ $+$\method 
& \textbf{15.2} & \textbf{54.8} 
& \textbf{12.7} & \textbf{57.1} 
& \textbf{45.3} & \textbf{42.7} 
& \textbf{21.5} & \textbf{48.0}\\
\bottomrule
\end{tabularx}

\vspace{0.35em}

\phantomsubcaption\label{tab:llama3_8b_results}
\panelcaption{(c) \texttt{LLaMA3-8B}, W2A16.}

\begin{tabularx}{\linewidth}{@{}l*{6}{Y}@{}}
\toprule
& \multicolumn{2}{c}{\texttt{30K}} 
& \multicolumn{2}{c}{\texttt{60K}} 
& \multicolumn{2}{c}{\texttt{90K}} \\
\cmidrule(lr){2-3} 
\cmidrule(lr){4-5} 
\cmidrule(lr){6-7}
Metrics 
& PPL & Acc. 
& PPL & Acc. 
& PPL & Acc.\\
\midrule
\algn{ParetoQ} 
& 11.78 & 58.54 
& 10.43 & 60.67 
& 9.79 & 61.41 \\
\ \ $+$\method 
& \textbf{11.69} & \textbf{59.10} 
& \textbf{10.32} & \textbf{60.98} 
& \textbf{9.71} & \textbf{62.04} \\
\bottomrule
\end{tabularx}
\end{minipage}

\end{table}

\paragraph{Main QAT results.}
\Cref{tab:main_results} shows that replacing the identity-\ste backward rule with \method improves two strong QAT pipelines under the same forward quantizer.  On \texttt{LLaMA3-1B}, \algn{ParetoQ}+\method improves PPL/Acc. from 17.9/51.9 to 17.6/52.4 at W1A16, from 14.8/54.3 to 13.9/55.3 at W1.58A16, and from 13.5/55.7 to 12.3/56.6 at W2A16.  The matched \algn{Rotation Trick} row gives a direct beyond-\ste comparison: \method is better at all three bit-widths, with the largest gap at W2A16 (13.3/56.1 $\rightarrow$ 12.3/56.6).  \algn{WinQ}+\method shows the same trend, reaching the best mean accuracy in the table (56.8).  PTQ baselines degrade sharply at 1--2 bits, confirming that this is a genuinely difficult ultra-low-bit regime.

\paragraph{Scaling to other families and the 8B run.}
\Cref{tab:complimentary_results} extends the matched-forward comparison to \texttt{LLaMA3-3B} and \texttt{Qwen3-1.7B}: \method improves both PPL and mean accuracy at W1A8/W1.58A8 for \texttt{LLaMA3-3B} and at W1A8/W2A8 for \texttt{Qwen3-1.7B}.  \Cref{tab:llama3_8b_results} adds the \texttt{LLaMA3-8B} W2A16 trajectory.  Across 30K/60K/90K checkpoints, \method monotonically improves ParetoQ from 11.78/58.54 to 11.69/59.10, from 10.43/60.67 to 10.32/60.98, and from 9.79/61.41 to 9.71/62.04 (PPL/Acc.).  These gains are modest but consistent, suggesting that the learned backward sensitivity remains useful beyond the smaller matched-ablation suite.

\begin{figure}
    \centering
    \includegraphics[width=1.0\linewidth]{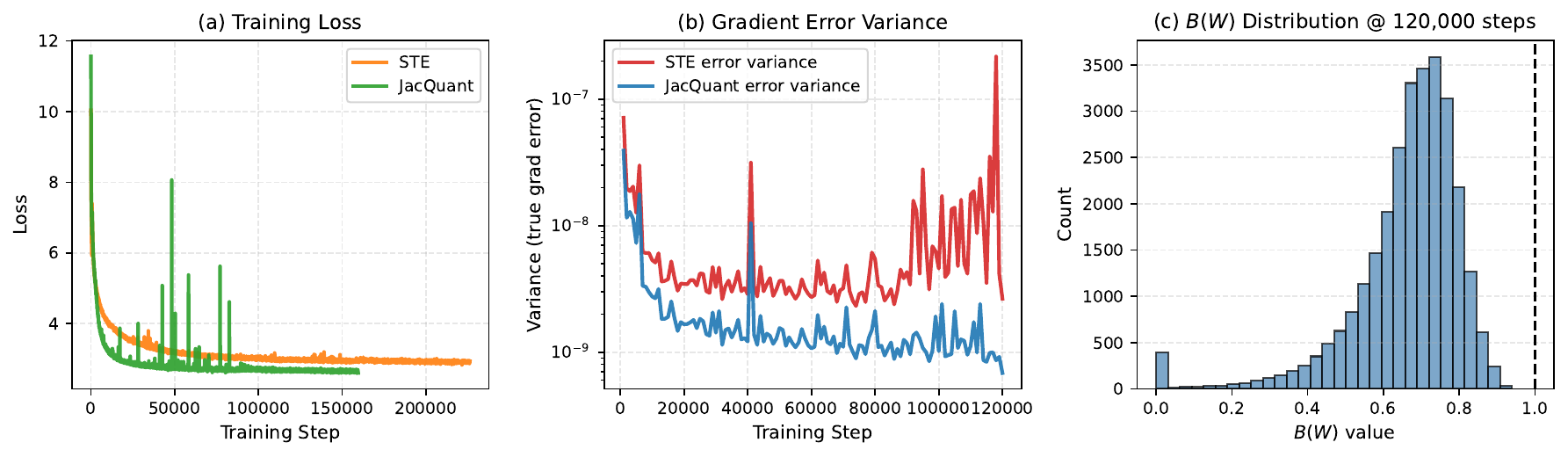}
    \vspace{-6mm}
\caption{\emph{Finite-difference gradient diagnostics.}
    (a) \method converges faster than \ste and stabilizes within a smaller neighborhood;
    (b) variance of the coordinate-wise mismatch between a central-difference reference gradient and the approximate training gradients;
    (c) histogram of the learned group-wise Jacobian scalars $b_g$ at the target step, with the dashed line indicating the \ste identity ($b=1$).}
    \label{fig:truegrad}
    \vspace{-4mm}
\end{figure}

\paragraph{Mechanistic diagnostic.}
Because hard quantizers are non-differentiable, we use finite differences only as a diagnostic reference: for sampled coordinates, we estimate $\hat J_i^{\mathrm{FD}}=(Q_\Delta(w_i+\epsilon)-Q_\Delta(w_i-\epsilon))/(2\epsilon)$ and compare $\hat J_i^{\mathrm{FD}}v_i$ with each training rule.  \Cref{fig:truegrad} shows that \method has lower coordinate-wise gradient-error variance than \ste and learns group gains below one in locally insensitive/clipped regions.  The small loss transients are expected because \method periodically refreshes $B(W)$; they do not contradict the diagnostic or final metrics, which both favor the learned surrogate.

\paragraph{Efficiency and ablations.}
\method stores one scalar sensitivity per quantization group (15--30MB in our LLM settings, $<1.5\%$ of Adam-state memory) and refreshes it only every 50--200 steps.  The runtime appendix (App. \ref{sec:runtime_w2a16}) reports both the \texttt{LLaMA3-1B} W2A16 throughput suite and the \texttt{LLaMA3-8B} W2A16 step-time table.  In the 1B suite, overhead is typically $\sim$1--5\% depending on group size and estimator; on \texttt{LLaMA3-8B}, practical Probe settings with $g_s\in\{32,128\}$ stay within $0.8\%$ of the \ste baseline, and Dither remains within roughly $1.6\%$.  \Cref{tab:ablation-gs-uf} further shows a broad near-optimal region rather than fragile tuning: e.g., $g_s=32$/uf=200 gives 12.3 PPL and 56.6 Avg. Acc., while $g_s=128$/uf=100 or 200 remains close at 12.3--12.5 PPL and 56.3--56.5 Avg. Acc.  This supports the design choice of amortized, moderately coarse sensitivity learning rather than aggressive per-step refresh.

\vspace{-2mm}
\section{Discussion and Future Work}
\vspace{-2mm}
\label{sec:discussion}
The experiments substantiate the theoretical picture in \Cref{sec:convergence}: ultra-low-bit QAT is bottlenecked not only by the forward quantizer, but also by the backward model used to train through it.  Across \texttt{LLaMA3} and \texttt{Qwen3} families, the \texttt{LLaMA3-8B} W2A16 trajectory, and the matched \algn{Rotation Trick} comparison, \method improves perplexity and zero-shot accuracy without changing the deployed hard quantizer.  Mechanistically, the learned $B(W)$ acts as an adaptive gain controller: it stays near identity in responsive bin interiors and damps saturated/clipped directions where the mean-field sensitivity is smaller than one.

The practical claim is therefore deliberately scoped.  \method is a drop-in \ste-free backward rule for ultra-low-bit LLM QAT, with end-to-end empirical evidence up to \texttt{LLaMA3-8B} and runtime measurements on both \texttt{LLaMA3-1B} W2A16 and \texttt{LLaMA3-8B} W2A16.  Larger end-to-end QAT runs remain an empirical direction, but the current measurements indicate that the added state and throughput cost of learning $B(W)$ are small under practical group sizes.  Limitations, non-uniform quantizer extensions, and additional ablations are discussed in App. \ref{app:limitations}.



\bibliographystyle{mbst}
\bibliography{mybib}

@article{zhou2016dorefa,
  title        = {DoReFa-Net: Training Low Bitwidth Convolutional Neural Networks with Low Bitwidth Gradients},
  author       = {Zhou, Shuchang and Wu, Yuxin and Ni, Zekun and Zhou, Xinyu and Wen, He and Zou, Yuheng},
  journal      = {arXiv preprint arXiv:1606.06160},
  year         = {2016},
  url          = {https://arxiv.org/abs/1606.06160}
}

@article{choi2018pact,
  title        = {PACT: Parameterized Clipping Activation for Quantized Neural Networks},
  author       = {Jungwook Choi and Zhuo Wang and Swagath Venkataramani and Pierce I.-Jen Chuang and Vijayalakshmi Srinivasan and Kailash Gopalakrnan},
  journal      = {arXiv preprint arXiv:1805.06085},
  year         = {2018},
  url          = {https://arxiv.org/abs/1805.06085}
}

@article{liu2025muon,
  title={Muon is scalable for LLM training},
  author={Liu, Jingyuan and Su, Jianlin and Yao, Xingcheng and Jiang, Zhejun and Lai, Guokun and Du, Yulun and Qin, Yidao and Xu, Weixin and Lu, Enzhe and Yan, Junjie and others},
  journal={arXiv preprint arXiv:2502.16982},
  year={2025}
}

@inproceedings{SpinQuant,
  title={SpinQuant: LLM Quantization with Learned Rotations},
  author={Liu, Zechun and Zhao, Changsheng and Fedorov, Igor and Soran, Bilge and Choudhary, Dhruv and Krishnamoorthi, Raghuraman and Chandra, Vikas and Tian, Yuandong and Blankevoort, Tijmen},
  booktitle={The Thirteenth International Conference on Learning Representations},
  year={2025}
}

@inproceedings{karimi2016linear,
  title={Linear convergence of gradient and proximal-gradient methods under the polyak-{\l}ojasiewicz condition},
  author={Karimi, Hamed and Nutini, Julie and Schmidt, Mark},
  booktitle={Joint European conference on machine learning and knowledge discovery in databases},
  pages={795--811},
  year={2016},
  organization={Springer}
}

@article{agustsson2017soft,
  title={Soft-to-hard vector quantization for end-to-end learning compressible representations},
  author={Agustsson, Eirikur and Mentzer, Fabian and Tschannen, Michael and Cavigelli, Lukas and Timofte, Radu and Benini, Luca and Gool, Luc V},
  journal={Advances in neural information processing systems},
  volume={30},
  year={2017}
}

@article{yang2025fogzo,
  title        = {Improving the Straight-Through Estimator with Zeroth-Order Information},
  author       = {Yang, Ningfeng and Aamodt, Tor M.},
  journal      = {arXiv preprint arXiv:2510.23926},
  year         = {2025}
}

@article{winq2025,
  title={WinQ: Accelerating Quantization-Aware Training of LLMs around Saddle Points},
  author={Dongyue Li and Zechun Liu and Kai Yi and Changsheng Zhao and Raghuraman Krishnamoorthi and Harshit Khaitan and Hongyang R. Zhang and Steven Li},
  journal={ICML},
  year={2026}
}

@article{gower2021stochastic,
  title={Stochastic quasi-gradient methods: Variance reduction via Jacobian sketching},
  author={Gower, Robert M and Richt{\'a}rik, Peter and Bach, Francis},
  journal={Mathematical Programming},
  volume={188},
  number={1},
  pages={135--192},
  year={2021},
  publisher={Springer}
}

@article{zeroqat,
  title={End-to-End On-Device Quantization-Aware Training for LLMs at Inference Cost},
  author={Tan, Qitao and Song, Xiaoying and Lu, Jin and Li, Guoming and Liu, Jun and Hong, Lingzi and Ding, Caiwen and Li, Jundong and Zhai, Xiaoming and Huang, Shaoyi and Niu, Wei and Yuan, Geng},
  journal={arXiv preprint arXiv:2509.00031},
  year={2025}
}

@article{llama3,
  title={The llama 3 herd of models},
  author={Grattafiori, Aaron and Dubey, Abhimanyu and Jauhri, Abhinav and Pandey, Abhinav and Kadian, Abhishek and Al-Dahle, Ahmad and Letman, Aiesha and Mathur, Akhil and Schelten, Alan and Vaughan, Alex and others},
  journal={arXiv preprint arXiv:2407.21783},
  year={2024}
}

@article{yang2025qwen3,
  title={Qwen3 technical report},
  author={Yang, An and Li, Anfeng and Yang, Baosong and Zhang, Beichen and Hui, Binyuan and Zheng, Bo and Yu, Bowen and Gao, Chang and Huang, Chengen and Lv, Chenxu and others},
  journal={arXiv preprint arXiv:2505.09388},
  year={2025}
}

@article{aqlm,
  title={Extreme compression of large language models via additive quantization},
  author={Egiazarian, Vage and Panferov, Andrei and Kuznedelev, Denis and Frantar, Elias and Babenko, Artem and Alistarh, Dan},
  journal={arXiv preprint arXiv:2401.06118},
  year={2024}
}

@article{malinovskii2024pv,
  title={Pv-tuning: Beyond straight-through estimation for extreme llm compression},
  author={Malinovskii, Vladimir and Mazur, Denis and Ilin, Ivan and Kuznedelev, Denis and Burlachenko, Konstantin and Yi, Kai and Alistarh, Dan and Richtarik, Peter},
  journal={Advances in Neural Information Processing Systems},
  volume={37},
  pages={5074--5121},
  year={2024}
}

@article{bengio2013estimating,
  title={Estimating or propagating gradients through stochastic neurons for conditional computation},
  author={Bengio, Yoshua and L{\'e}onard, Nicholas and Courville, Aaron},
  journal={arXiv preprint arXiv:1308.3432},
  year={2013}
}

@inproceedings{liu2024vptq,
  title={VPTQ: Extreme Low-bit Vector Post-Training Quantization for Large Language Models},
  author={Liu, Yifei and Wen, Jicheng and Wang, Yang and Ye, Shengyu and Zhang, Li Lyna and Cao, Ting and Li, Cheng and Yang, Mao},
  booktitle={Proceedings of the 2024 Conference on Empirical Methods in Natural Language Processing},
  pages={8181--8196},
  year={2024}
}

@article{malinovsky2022variance,
  title={Variance reduced proxskip: Algorithm, theory and application to federated learning},
  author={Malinovsky, Grigory and Yi, Kai and Richt{\'a}rik, Peter},
  journal={Advances in Neural Information Processing Systems},
  volume={35},
  pages={15176--15189},
  year={2022}
}

@inproceedings{liu2025paretoq,
  title={ParetoQ: Improving Scaling Laws in Extremely Low-bit LLM Quantization},
  author={Liu, Zechun and Zhao, Changsheng and Huang, Hanxian and Chen, Sijia and Zhang, Jing and Zhao, Jiawei and Roy, Scott and Jin, Lisa and Xiong, Yunyang and Shi, Yangyang and Xiao, Lin and Tian, Yuandong and Soran, Bilge and Krishnamoorthi, Raghuraman and Blankevoort, Tijmen and Chandra, Vikas},
  booktitle={Advances in Neural Information Processing Systems},
  year={2025}
}

@article{chee2023quip,
  title={Quip: 2-bit quantization of large language models with guarantees},
  author={Chee, Jerry and Cai, Yaohui and Kuleshov, Volodymyr and De Sa, Christopher M},
  journal={Advances in Neural Information Processing Systems},
  volume={36},
  pages={4396--4429},
  year={2023}
}

@inproceedings{bhalgat2020lsq+,
  title={Lsq+: Improving low-bit quantization through learnable offsets and better initialization},
  author={Bhalgat, Yash and Lee, Jinwon and Nagel, Markus and Blankevoort, Tijmen and Kwak, Nojun},
  booktitle={Proceedings of the IEEE/CVF conference on computer vision and pattern recognition workshops},
  pages={696--697},
  year={2020}
}

@article{johnson2013accelerating,
  title={Accelerating stochastic gradient descent using predictive variance reduction},
  author={Johnson, Rie and Zhang, Tong},
  journal={Advances in neural information processing systems},
  volume={26},
  year={2013}
}

@article{fang2018spider,
  title={Spider: Near-optimal non-convex optimization via stochastic path-integrated differential estimator},
  author={Fang, Cong and Li, Chris Junchi and Lin, Zhouchen and Zhang, Tong},
  journal={Advances in neural information processing systems},
  volume={31},
  year={2018}
}

@article{clark2018think,
  title={Think you have Solved Question Answering? Try ARC, the AI2 Reasoning Challenge},
  author={Clark, Peter and Cowhey, Isaac and Etzioni, Oren and Khot, Tushar and Sabharwal, Ashish and Schoenick, Carissa and Tafjord, Oyvind},
  journal={arXiv preprint arXiv:1803.05457},
  year={2018}
}

@inproceedings{clark2019boolq,
  title={BoolQ: Exploring the surprising difficulty of natural yes/no questions},
  author={Clark, Christopher and Lee, Kenton and Chang, Ming-Wei and Kwiatkowski, Tom and Collins, Michael and Toutanova, Kristina},
  booktitle={Proceedings of the 2019 Conference of the North American Chapter of the Association for Computational Linguistics: Human Language Technologies, Volume 1 (Long and Short Papers)},
  pages={2924--2936},
  year={2019}
}

@inproceedings{bisk2020piqa,
  title={PIQA: Reasoning about physical commonsense in natural language},
  author={Bisk, Yonatan and Zellers, Rowan and Gao, Jianfeng and Choi, Yejin and others},
  booktitle={Proceedings of the AAAI Conference on Artificial Intelligence},
  volume={34},
  number={05},
  pages={7432--7439},
  year={2020}
}

@inproceedings{sap2019socialiqa,
  title={SocialIQA: Commonsense reasoning about social interactions},
  author={Sap, Maarten and Rashkin, Hannah and Chen, Derek and LeBras, Ronan and Choi, Yejin},
  booktitle={Proceedings of the 2019 Conference on Empirical Methods in Natural Language Processing and the 9th International Joint Conference on Natural Language Processing (EMNLP-IJCNLP)},
  pages={4463--4473},
  year={2019}
}

@inproceedings{zellers2019hellaswag,
  title={HellaSwag: Can a machine really finish your sentence?},
  author={Zellers, Rowan and Holtzman, Ari and Bisk, Yonatan and Farhadi, Ali and Choi, Yejin},
  booktitle={Proceedings of the 57th Annual Meeting of the Association for Computational Linguistics},
  pages={4791--4800},
  year={2019}
}

@inproceedings{mihaylov2018can,
  title={Can a suit of armor conduct electricity? A new dataset for open book question answering},
  author={Mihaylov, Todor and Clark, Peter and Khot, Tushar and Sabharwal, Ashish},
  booktitle={Proceedings of the 2018 Conference on Empirical Methods in Natural Language Processing},
  pages={2381--2391},
  year={2018}
}

@inproceedings{sakaguchi2021winogrande,
  title={WinoGrande: An adversarial winograd schema challenge at scale},
  author={Sakaguchi, Keisuke and Bras, Ronan Le and Bhagavatula, Chandra and Choi, Yejin},
  booktitle={Proceedings of the AAAI Conference on Artificial Intelligence},
  volume={34},
  number={05},
  pages={8732--8740},
  year={2020}
}

@inproceedings{merity2017pointer,
  title={Pointer Sentinel Mixture Models},
  author={Merity, Stephen and Xiong, Caiming and Bradbury, James and Socher, Richard},
  booktitle={International Conference on Learning Representations},
  year={2017}
}

@article{penedo2024fineweb,
  title={The fineweb datasets: Decanting the web for the finest text data at scale},
  author={Penedo, Guilherme and Kydl{\'\i}{\v{c}}ek, Hynek and Lozhkov, Anton and Mitchell, Margaret and Raffel, Colin A and Von Werra, Leandro and Wolf, Thomas and others},
  journal={Advances in Neural Information Processing Systems},
  volume={37},
  pages={30811--30849},
  year={2024}
}

@article{cutkosky2019momentum,
  title={Momentum-based variance reduction in non-convex sgd},
  author={Cutkosky, Ashok and Orabona, Francesco},
  journal={Advances in neural information processing systems},
  volume={32},
  year={2019}
}

@article{jiang2024adaptive,
  title={Adaptive variance reduction for stochastic optimization under weaker assumptions},
  author={Jiang, Wei and Yang, Sifan and Wang, Yibo and Zhang, Lijun},
  journal={Advances in Neural Information Processing Systems},
  volume={37},
  pages={22047--22080},
  year={2024}
}

@article{li2021zerosarah,
  title={ZeroSARAH: Efficient nonconvex finite-sum optimization with zero full gradient computation},
  author={Li, Zhize and Hanzely, Slavom{\'\i}r and Richt{\'a}rik, Peter},
  journal={arXiv preprint arXiv:2103.01447},
  year={2021}
}

@inproceedings{xiao2023smoothquant,
  title={Smoothquant: Accurate and efficient post-training quantization for large language models},
  author={Xiao, Guangxuan and Lin, Ji and Seznec, Mickael and Wu, Hao and Demouth, Julien and Han, Song},
  booktitle={International conference on machine learning},
  pages={38087--38099},
  year={2023},
  organization={PMLR}
}

@inproceedings{nguyen2017sarah,
  title={SARAH: A novel method for machine learning problems using stochastic recursive gradient},
  author={Nguyen, Lam M and Liu, Jie and Scheinberg, Katya and Tak{\'a}{\v{c}}, Martin},
  booktitle={International conference on machine learning},
  pages={2613--2621},
  year={2017},
  organization={PMLR}
}

@article{roux2012stochastic,
  title={A Stochastic Gradient Method with an Exponential Convergence Rate for Finite Training Sets},
  author={Roux, Nicolas Le and Schmidt, Mark and Bach, Francis},
  journal={Advances in Neural Information Processing Systems},
  volume={25},
  year={2012}
}

@article{lin2024awq,
  title={Awq: Activation-aware weight quantization for on-device llm compression and acceleration},
  author={Lin, Ji and Tang, Jiaming and Tang, Haotian and Yang, Shang and Chen, Wei-Ming and Wang, Wei-Chen and Xiao, Guangxuan and Dang, Xingyu and Gan, Chuang and Han, Song},
  journal={Proceedings of machine learning and systems},
  volume={6},
  pages={87--100},
  year={2024}
}

@article{yao2022zeroquant,
  title={Zeroquant: Efficient and affordable post-training quantization for large-scale transformers},
  author={Yao, Zhewei and Yazdani Aminabadi, Reza and Zhang, Minjia and Wu, Xiaoxia and Li, Conglong and He, Yuxiong},
  journal={Advances in neural information processing systems},
  volume={35},
  pages={27168--27183},
  year={2022}
}

@article{yao2023zeroquant,
  title={Zeroquant-v2: Exploring post-training quantization in llms from comprehensive study to low rank compensation},
  author={Yao, Zhewei and Wu, Xiaoxia and Li, Cheng and Youn, Stephen and He, Yuxiong},
  journal={arXiv preprint arXiv:2303.08302},
  year={2023}
}

@article{svrg,
  title={Variance reduced stochastic gradient descent with neighbors},
  author={Hofmann, Thomas and Lucchi, Aurelien and Lacoste-Julien, Simon and McWilliams, Brian},
  journal={Advances in Neural Information Processing Systems},
  volume={28},
  year={2015}
}

@inproceedings{rotationtrick,
  title={Restructuring Vector Quantization with the Rotation Trick},
  author={Fifty, Christopher and Junkins, Ronald Guenther and Duan, Dennis and Iyengar, Aniketh and Liu, Jerry Weihong and Amid, Ehsan and Thrun, Sebastian and Re, Christopher},
  booktitle={The Thirteenth International Conference on Learning Representations},
  year={2025}
}

@article{defazio2014saga,
  title={SAGA: A fast incremental gradient method with support for non-strongly convex composite objectives},
  author={Defazio, Aaron and Bach, Francis and Lacoste-Julien, Simon},
  journal={Advances in neural information processing systems},
  volume={27},
  year={2014}
}

@inproceedings{nagel2020up,
  title={Up or down? adaptive rounding for post-training quantization},
  author={Nagel, Markus and Amjad, Rana Ali and Van Baalen, Mart and Louizos, Christos and Blankevoort, Tijmen},
  booktitle={International conference on machine learning},
  pages={7197--7206},
  year={2020},
  organization={PMLR}
}

@inproceedings{esser2020learned,
  title={Learned Step Size Quantization},
  author={Esser, Steven K and McKinstry, Jeffrey L and Bablani, Deepika and Appuswamy, Rathinakumar and Modha, Dharmendra S},
  booktitle={International Conference on Learning Representations},
  year={2020}
}

@inproceedings{yin2019understanding,
  title={Understanding straight-through estimator in training activation quantized neural nets},
  author={Yin, P and Lyu, J and Zhang, S and Osher, S and Qi, YY and Xin, J},
  booktitle={International Conference on Learning Representations},
  year={2019}
}

@inproceedings{gptq,
  title={GPTQ: Accurate Quantization for Generative Pre-trained Transformers},
  author={Frantar, Elias and Ashkboos, Saleh and Hoefler, Torsten and Alistarh, Dan},
  booktitle={The Eleventh International Conference on Learning Representations},
  year={2023}
}

@inproceedings{tseng2024quip,
  title={QuIP $\# $: Even Better LLM Quantization with Hadamard Incoherence and Lattice Codebooks},
  author={Tseng, Albert and Chee, Jerry and Sun, Qingyao and Kuleshov, Volodymyr and De Sa, Christopher},
  booktitle={International Conference on Machine Learning},
  pages={48630--48656},
  year={2024},
  organization={PMLR}
}

\appendix
\newpage
\tableofcontents
\newpage

\appendix

\section{Limitations and Future Work}\label{app:limitations}
While \method already improves stability and final quality in $\le$2-bit LLM QAT, several directions remain open. First, extending Jacobian learning to \textbf{mixed-precision policies} and more aggressive \textbf{activation quantization} (including dynamic-activation and heterogeneous regimes) would broaden applicability beyond the weight-centric configurations emphasized here; a natural approach is to learn coupled sensitivity surrogates for weights, activations, and clipping/scale parameters. Second, \textbf{broader scaling} remains an important empirical direction. The present paper includes an 8B W2A16 accuracy trajectory and explicit runtime measurements on both \texttt{LLaMA3-1B} W2A16 and \texttt{LLaMA3-8B} W2A16; still larger end-to-end QAT runs would further test how far the learned-backward-rule advantage persists. Third, there is room to improve the \textbf{estimation and scheduling} of $B(W)$: although Probe currently outperforms Dither in our implementation (Appendix~\ref{app:ablation-gs-uf}), refined dither statistics, adaptive probe scales, and refresh criteria tied to quantization drift could improve robustness under frequent code changes. Finally, on the theory side, sharpening \textbf{non-convex guarantees beyond PL} and relaxing reliance on code-preserving windows to better capture frequent bin flips in transformer training remains an important challenge; addressing it would deepen understanding of why and when learned sensitivity maps most effectively stabilize quantized optimization.

\section{Broader Impacts}\label{app:broader-impacts}
This work studies a training algorithm for low-bit LLM quantization rather than a new model, dataset, or deployment product. Its main positive impact is efficiency: more stable ultra-low-bit QAT can reduce memory footprint and training/inference cost, which may lower energy use and make model compression research more accessible. The same efficiency gains may also make capable language models easier to deploy, including in applications with potential misuse or fairness, privacy, and safety concerns. \method does not introduce new model capabilities or collect new data, so these risks primarily inherit from the underlying models and deployment contexts. Responsible use should therefore follow the licenses, safety policies, evaluation protocols, and monitoring practices associated with the base models and downstream applications.

\section{Extended Related Work}
\label{app:relwork}

\subsection{LLM Quantization}
\textbf{Post-Training Quantization (PTQ).}
PTQ avoids gradient updates by calibrating a pretrained model with a small dataset.
For LLMs, \algn{SmoothQuant}~\citep{xiao2023smoothquant} smooths activation outliers via an offline migration of scale
from activations to weights, enabling accurate W8A8 deployment across many LLMs.
\algn{GPTQ}~\citep{gptq} introduces a highly accurate, per-layer Hessian-informed weight quantizer that attains 3--4 bit
weight-only compression with strong perplexity retention.
\algn{ZeroQuant}/\algn{ZeroQuant-V2}~\citep{yao2022zeroquant,yao2023zeroquant} propose affordable layer-by-layer PTQ
pipelines with hardware-friendly kernels and calibration-friendly deployment.
\algn{AWQ}~\citep{lin2024awq} further shows that protecting a small set of salient channels---identified by activations rather
than weights---reduces quantization error and improves practical 4-bit inference.
Overall, PTQ has become a strong default for training-free LLM deployment.

\textbf{Quantization-Aware Training (QAT).}
Early QAT works established the use of the Straight-Through Estimator (\ste) for backpropagating through non-differentiable
quantizers.
\algn{DoReFa-Net}~\citep{zhou2016dorefa} quantizes weights, activations, and gradients during training;
\algn{PACT}~\citep{choi2018pact} learns activation clipping levels end-to-end;
and \algn{LSQ}/\algn{LSQ+}~\citep{esser2020learned,bhalgat2020lsq+} learn per-layer step sizes, narrowing the accuracy gap
at low bitwidths.
Despite empirical success, most QAT methods inherit \ste’s bias, which ignores discrete bin geometry and can induce
optimization mismatch at 2--4 bits.
Recent \ste-free directions include proximal surrogates (e.g., \algn{PV-Tuning}~\citep{malinovskii2024pv}) and geometric
transforms (e.g., the \algn{Rotation Trick}~\citep{rotationtrick}). Forward-only zeroth-order QAT methods such as \algn{ZeroQAT}~\citep{zeroqat} attack a complementary scalability bottleneck by changing how gradients are estimated; \method instead replaces only the backward map used with the deployed hard quantizer.

\textbf{Why QAT for low-bit LLMs.}
Consistent with recent observations (e.g., \algn{ParetoQ}~\citep{liu2025paretoq}, \algn{WinQ}~\citep{winq2025}), QAT typically
outperforms PTQ under aggressive settings ($\leq$2-bit weights and/or W\&A quantization), because the optimizer can adapt
layer scales and clipping behavior to the task distribution rather than relying solely on calibration statistics.
Motivated by this, our work focuses on QAT but replaces the fixed \ste with a \emph{learned} surrogate Jacobian that adapts
to local quantization geometry.

\subsection{Beyond the Straight-Through Estimator}\label{app:beyond_ste}
Most QAT methods still rely on the Straight-Through Estimator (\ste) or its clipped variants to approximate the non-differentiable backward path~\citep{esser2020learned,bhalgat2020lsq+,gptq,chee2023quip,tseng2024quip}. Recent work explores several strategies to mitigate \ste mismatch: (i) analytic/proximal backward surrogates (e.g., \algn{PV-Tuning}~\citep{malinovskii2024pv}), (ii) fixed geometric transforms (e.g., \algn{Rotation Trick}~\citep{rotationtrick}), (iii) differentiable relaxations and continuation schedules (e.g., \algn{Soft-to-Hard VQ}~\citep{agustsson2017soft}), and (iv) hybrid first/zeroth-order estimators that use finite-difference probes to correct \ste bias (e.g., \algn{FOGZO}~\citep{yang2025fogzo}). In parallel, extreme low-bit LLM methods may modify the optimizer geometry (e.g., \algn{Muon}~\citep{liu2025muon}).

\textbf{Why Jacobian learning?}
We target the core failure mode shared by these settings: the backward pass lacks a faithful local sensitivity model for the
(piecewise-constant) quantizer. JacQuant therefore \emph{keeps the hard quantizer in the forward pass} but learns a lightweight
surrogate Jacobian $B(W)$ to approximate the quantizer's mean-field sensitivity, reducing gradient bias exactly where \ste is
most problematic (near clipping/saturation) while remaining a drop-in replacement with negligible overhead.

\textbf{JacQuant vs. PV-Tuning.}
\algn{PV-Tuning}~\citep{malinovskii2024pv} eliminates \ste by reformulating quantization as a proximal optimization problem. Instead of copying gradients through the quantizer, it introduces an analytic surrogate via a proximal operator with explicit convergence guarantees. However, the resulting Jacobian is \emph{static}: it depends on the proximal regularization and does not adapt to local loss curvature or quantization geometry. In contrast, \method learns a \emph{data-driven surrogate Jacobian} $B(W)$ that captures backward sensitivity through stochastic probing or dithering, dynamically aligning gradient propagation with the quantized landscape. As a result, \method can adapt its backward gain across groups/layers and over time (e.g., damping saturated coordinates), improving stability and convergence in ultra-low-bit regimes without changing the forward quantizer.

\textbf{JacQuant vs. Rotation Trick.}
The \algn{Rotation Trick}~\citep{rotationtrick} mitigates \ste bias via a fixed geometric constraint (preserving an angle relationship), but still uses a hand-crafted static transform. In contrast, \method learns an adaptive surrogate Jacobian $B(W)$ that models how perturbations in weights affect quantized parameters, generalizing fixed rotations as a special case. This adaptivity is crucial when quantization sensitivity is heterogeneous across layers/groups, where a single fixed transform cannot simultaneously correct boundary effects and saturation behavior. The matched \algn{ParetoQ} comparison in \Cref{tab:main_results} makes this distinction quantitative: \method improves over \algn{Rotation Trick} at W1A16, W1.58A16, and W2A16, with the largest PPL gap at W2A16 (13.3$\rightarrow$12.3).

\textbf{JacQuant vs. Soft-to-Hard VQ.}
\algn{Soft-to-Hard VQ}~\citep{agustsson2017soft} optimizes a \emph{smoothed} relaxation of quantization and gradually anneals toward hard assignments, enabling gradients to flow through the relaxed objective early in training. This is effective when a continuation schedule is acceptable and when the relaxed objective closely tracks the deployed hard quantizer.
\emph{Why we do not rely on relaxation:}
for low-bit LLM QAT, the forward operator used at deployment is typically a fixed hard quantizer, and a soft relaxation can introduce a train--deploy mismatch that depends on the annealing schedule and temperature.
\emph{JacQuant's benefit:}
JacQuant keeps the \emph{exact hard quantizer} in the forward pass and instead learns a local backward sensitivity $B(W)$, which directly corrects gradient bias near bin boundaries and clipping/saturation while preserving forward faithfulness. The methods are complementary: relaxation can serve as a warm start, while JacQuant improves the final hard-quantized optimization by learning the missing backward model.

\textbf{JacQuant vs. FOGZO (zeroth-order information).}
\algn{FOGZO}~\citep{yang2025fogzo} improves upon \ste by injecting zeroth-order (finite-difference) information, producing a hybrid estimator that can correct \ste when the straight-through gradient is unreliable.
\emph{Why we do not use per-step zeroth-order correction:}
finite-difference corrections typically require additional forward evaluations and must be repeated frequently to remain accurate as weights drift.
\emph{JacQuant's benefit:}
JacQuant amortizes sensitivity estimation by learning a reusable surrogate Jacobian $B(W)$ (diagonal or block-diagonal) and updating it only occasionally (e.g., on refresh steps), while using $B(W)$ to modulate standard backprop gradients on all other steps. This yields a principled bias/variance improvement with minimal runtime overhead, rather than paying repeated zeroth-order query cost at every iteration.

\textbf{JacQuant vs. ZeroQAT.}
\algn{ZeroQAT}~\citep{zeroqat} targets a different operating point: it uses forward-only zeroth-order estimation to reduce backpropagation memory and time costs. \method instead stays inside standard first-order QAT and changes only the quantizer backward rule. Thus, ZeroQAT is a systems/estimation alternative for reducing QAT cost, whereas \method is a drop-in gradient-surrogate replacement that can in principle be combined with memory-saving QAT recipes.

\textbf{JacQuant vs. Muon.}
\algn{Muon}~\citep{liu2025muon} is an optimizer that improves the \emph{update geometry} for matrix parameters (e.g., orthogonalizing momentum) and can speed up or stabilize large-scale training in full precision.
\emph{Why we address the quantizer Jacobian:}
optimizer improvements do not resolve the fact that the gradient signal itself is biased when backpropagating through a hard quantizer with \ste.
\emph{JacQuant's benefit:}
\method corrects the \emph{gradient estimator} by learning the missing backward sensitivity map $B(W)$, and is therefore compatible with (and potentially enhanced by) better optimizers such as Muon. This separation of concerns lets us retain standard QAT pipelines and forward quantizers while directly targeting the \ste-induced mismatch.

\textbf{Applicability beyond uniform scalar quantizers.}
Our experiments focus on symmetric uniform LLM QAT because it is the cleanest setting for isolating the backward-rule contribution. The mechanism is nevertheless not tied to uniform bins: for separable non-uniform scalar quantizers, the same group-wise slope estimator can learn local sensitivities around non-equally-spaced thresholds; for vector, trellis, or codebook quantizers, the natural extension is to replace scalar group gains with richer block-structured $B(W)$ matching the quantizer block/codeword structure. These variants may require different probe covariance, block size, and refresh schedules, so we treat them as promising extensions rather than claims established by the present experiments.

These comparisons align with our analysis: \method explicitly learns $B(W)\approx J(W)$, reducing bias to a principled target gradient and naturally damping saturated directions, whereas many alternatives either keep a static surrogate, change the forward objective, or pay repeated zeroth-order query cost.

\subsection{Variance Reduction in Stochastic Optimization}\label{app:vr}
Variance-reduced (\algn{VR}) methods accelerate stochastic optimization by reducing gradient noise via control variates, including
\algn{SAG}~\citep{roux2012stochastic}, \algn{SVRG}~\citep{johnson2013accelerating,malinovsky2022variance},
and \algn{SAGA}~\citep{defazio2014saga}, with extensions such as
\algn{SARAH}~\citep{nguyen2017sarah,li2021zerosarah},
\algn{SPIDER}~\citep{fang2018spider}, and \algn{STORM}~\citep{cutkosky2019momentum,jiang2024adaptive}.
Most of these works assume a (possibly noisy) \emph{fixed} gradient oracle for a smooth objective and focus on shrinking the
variance of minibatch gradients while preserving convergence rates.
A practical distinction is that some classical \algn{VR} schemes (notably \algn{SAG}/\algn{SAGA}) rely on \emph{per-example memory}
(control variates stored for all data indices), which becomes impractical for LLM-scale corpora and streaming training where
the effective dataset size is massive.

\textbf{Why \algn{VR} is nontrivial for \ste-free QAT.}
In \method, stochasticity arises not only from minibatch sampling, but also from \emph{learning the backward model itself}:
we optimize using per-sample surrogate gradients
$F_i(W) := B(W)\,v_i$ where $v_i=\nabla_{q}\ell(f(x_i;q),y_i)$ and $q=Q_\Delta(W)$, while $B(W)$ is estimated by
stochastic probing/dithering and updated intermittently.
This introduces an additional source of noise and drift beyond standard SGD/Adam, and naive frequent Jacobian updates can
inject high-frequency variance that destabilizes training near quantization boundaries.

\textbf{JacSketch connection and our adaptation (scalable \algn{VR} for LLMs).}
\algn{JacSketch}~\citep{gower2021stochastic} interprets \algn{VR} as \emph{Jacobian sketching} for stochastic oracles, yielding
low-variance quasi-gradient updates.
A common and effective instantiation of this principle is \algn{SAGA}-style memory (maintaining a table of historical per-index
control variates), but such \emph{dataset-sized} memory is typically infeasible for LLM training.
\method adopts the control-variate principle but is designed for the LLM regime: we (i) do \emph{not} sketch the Jacobian of the loss;
(ii) instead maintain a lightweight diagonal/block-diagonal surrogate of the \emph{quantizer sensitivity} via $B(W)$; and
(iii) use a \emph{low-memory} \algn{VR} estimator.
In particular, our theory is \algn{VR}-agnostic: we analyze a general class of variance-reduced estimators via an \emph{ABC-type}
oracle assumption (cf.\ the ABC bounds used in our convergence analysis), which covers a broad family of control-variate and
anchor-based methods.
Motivated by implementability at LLM scale, we instantiate this framework with a lightweight \algn{L\mbox{-}SVRG} (loopless \algn{SVRG})-style estimator that requires no per-example table.

Concretely, letting $F_S(W,B):=\frac{1}{|S|}\sum_{i\in S}F_i(W)$, the \algn{L\mbox{-}SVRG} update takes the form
$$
g_t \;=\; F_{S_t}(W_t,B_t)\;-\;F_{S_t}(\tilde W,\tilde B)\;+\;\tilde g,
\qquad
\tilde g \;=\; \mathrm{RefGrad}(\tilde W,\tilde B),
$$
where $(\tilde W,\tilde B,\tilde g)$ is an anchor state computed on a (full or large) reference batch and refreshed
probabilistically.
This reduces variance of the \emph{\method surrogate gradient} without changing the forward quantizer, while avoiding the
$\mathcal{O}(n)$ memory footprint of \algn{SAGA}/\algn{JacSketch} tables.

\textbf{Our key design: amortized refresh enables cheap and stable Jacobian learning.}
Inspired by skip/refresh \algn{VR} frameworks (e.g., \algn{VR-ProxSkip}~\citep{malinovsky2022variance}), \method refreshes the anchor
state $(\tilde W,\tilde B,\tilde g)$ only occasionally (e.g., with probability $p$ or via a drift/variance criterion).
Importantly, we also update $B(W)$ primarily on these refresh steps.
This coupling is central in \method: it (i) amortizes the cost of probing/dithering so most iterations remain as cheap as
standard QAT, (ii) prevents noisy Jacobian estimates from dominating the optimization, and (iii) keeps the \algn{VR} machinery
\emph{LLM-friendly} by using \algn{L\mbox{-}SVRG} (constant-memory anchor state) rather than dataset-sized \algn{SAGA} memory.

\textbf{What is gained relative to ``\algn{VR} alone''.}
Variance reduction by itself only shrinks stochastic noise; it does not address the \ste-induced \emph{bias} caused by using an
incorrect backward Jacobian.
\method separates these effects: Jacobian learning reduces bias by driving $B(W)$ toward the mean-field sensitivity,
while \algn{VR} tightens the remaining variance constant of stochastic updates.
This combination yields a geometry-aware \algn{VR} paradigm tailored to quantized optimization: forward quantization remains
unchanged, while the backward path is both \emph{bias-corrected} (via $B(W)$) and \emph{variance-controlled} (via scalable \algn{VR}).

\section{Implementation Primitives for \jvr}
\label{app:vr-api}

This appendix specifies the minimal “API” used by \Cref{alg:jacquant-unified}.
Let $S\subset\{1,\ldots,n\}$ be a minibatch, $q=Q_\Delta(W)$ (or $q=Q_\Delta(W{+}r){-}r$ for \textsc{Dither}),
$v_i=\nabla_q\ell(f(x_i;q),y_i)$, and $F_i(W){=}B(W)\,v_i$.

\medskip\noindent
\textbf{\GradEst: General control-variate gradient estimator.}
$$
g = \underbrace{\tfrac{1}{|S|}\sum_{i\in S}\big(F_i(W)-h_i(s)\big)}_{\text{stochastic difference}} +  
\underbrace{\tilde g}_{\text{reference term}},
$$
where $s$ is the VR state (e.g., a table), $h_i(s)$ the per-index control variate,
and $\tilde g$ is a reference gradient at anchor $(\tilde W,\tilde B)$.

\emph{Typical instantiations (choose one):}
\begin{align*}
\textbf{\algn{SVRG}-like:}\quad
& h_i(s)=F_i(\tilde W),
\quad \tilde g=\RefGrad(\tilde W,\tilde B)=\tfrac{1}{n}\sum_{j=1}^n F_j(\tilde W). \\
\textbf{\algn{SAGA}-like:}\quad
& s=\{Y_1,\dots,Y_n,\ \bar Y=\tfrac{1}{n}\sum_j Y_j\},\
h_i(s)=Y_i,\ \tilde g=\bar Y. \\
\textbf{\algn{SARAH}-like:}\quad
& \text{refresh: } \tilde g=\tfrac{1}{|S_{\rm ref}|}\sum_{j\in S_{\rm ref}}F_j(W),\ g=\tilde g;\
\text{else: } g=\tfrac{1}{|S|}\sum_{i\in S}\big(F_i(W)-F_i(W^{-})\big)+g^{-}.
\end{align*}

\medskip\noindent
\textbf{\RefGrad: Anchor/reference gradient.}
$$
\RefGrad(\tilde W,\tilde B)
  =  
\tfrac{1}{|S_{\rm ref}|}\sum_{i\in S_{\rm ref}} \tilde B\,\nabla_q \ell\left(f(x_i;\tilde q),y_i\right),
\quad \tilde q=Q_\Delta(\tilde W)\ (\text{or }Q_\Delta(\tilde W{+}r){-}r).
$$

\medskip\noindent
\textbf{\CtrlUpdate: Update of the VR state $s$.}
$$
\CtrlUpdate(s,W,\tilde W;S):
\ \begin{cases}
\text{\algn{SAGA}-like: } Y_i \leftarrow F_i(W)\ \forall i\in S,\quad
\bar Y \leftarrow \bar Y+\tfrac{1}{n}\sum_{i\in S}\big(F_i(W)-Y_i^{\text{old}}\big); \\
\text{\algn{SVRG}-like: } \text{stateless; refresh sets }(\tilde W,\tilde B,\tilde g); \\
\text{\algn{SARAH}-like: } \text{maintain recursive }g\text{ and previous }(W^{-},g^{-}).
\end{cases}
$$

\medskip\noindent
\textbf{\ProbeUpdate: Probe-based Jacobian update.}
For each quantization group $g$ (size $d_g$), draw $\delta_g\sim\mathcal N(0,\sigma^2 I)$, measure
$\Delta q_g=Q_\Delta(W_g+\delta_g)-Q_\Delta(W_g)$, and apply an EMA least-squares step:
$$
\widehat B_g   =   \arg\min_{B\in\mathcal B_g}\ \|\Delta q_g - B\,\delta_g\|^2,\quad
B_g \leftarrow (1{-}\beta)B_g + \beta\,\widehat B_g,\ \ 
\mathcal B_g=\{\text{diag/block-diag}\}.
$$
Diagonal one-step: $[\widehat B_g]_{jj}=\frac{(\Delta q_g)_j\,\delta_{g,j}}{\delta_{g,j}^2+\epsilon}$.

\medskip\noindent
\textbf{\DitherEMA: Dither-based Jacobian update.}
With $r\sim\mathrm{Unif}[-\tfrac{\Delta}{2},\tfrac{\Delta}{2}]$, set $q=Q_\Delta(W{+}r)-r$ and update
$$
B \leftarrow (1{-}\beta)\,B + \beta\,\phi(W;r),
$$
where $\phi(\cdot)$ is a diagonal/block-diagonal sensitivity statistic of $(q,W)$ (damped near saturation).

\begin{algorithm}[t]
\caption{API reference for \jvr}
\label{alg:api-routines}
\small
\begin{algorithmic}[1]
\Function{\GradEst}{$W,B,\tilde W,\tilde B,\tilde g,s;S$}
  \State \textbf{return} $\frac{1}{|S|}\sum_{i\in S}\big(B\nabla_q\ell(f(x_i;q),y_i)-h_i(s)\big)+\tilde g$
\EndFunction
\Function{\RefGrad}{$\tilde W,\tilde B$}
  \State Sample $S_{\rm ref}$ (full or large minibatch)
  \State \textbf{return} $\frac{1}{|S_{\rm ref}|}\sum_{i\in S_{\rm ref}}\tilde B\,\nabla_q\ell(f(x_i;\tilde q),y_i)$
\EndFunction
\Function{\ProbeUpdate}{$W$}
  \For{each group $g$}
    \State Draw $\delta_g\sim\mathcal N(0,\sigma^2 I)$; $\Delta q_g\leftarrow Q_\Delta(W_g{+}\delta_g)-Q_\Delta(W_g)$
    \State $\widehat B_g\leftarrow\arg\min_{B\in\mathcal B_g}\|\Delta q_g{-}B\delta_g\|^2$ \ \ (diag/blk-diag LS)
    \State $B_g \leftarrow (1{-}\beta)B_g + \beta\,\widehat B_g$
  \EndFor
  \State \textbf{return} $B$
\EndFunction
\end{algorithmic}
\end{algorithm}

\section{Non-Variance-Reduced Variants of \algn{JacQuant}}
\label{app:jacquant-variants}

This appendix provides additional implementation details of the proposed Jacobian-learning framework, including the full pseudocode for the  non-variance-reduced (Base) variants (the main variance-reduced variant is provided in \Cref{alg:jacquant-unified}), as well as a summary table comparing their key characteristics.

\textbf{\jvr (Main Algorithm).}
\Cref{alg:jacquant-unified} presents the unified variance-reduced version of \algn{JacQuant}, which integrates the stochastic control-variate principle of \algn{JacSketch}/\algn{SAGA} into quantization-aware training.
The algorithm maintains a table of historical Jacobian-modulated gradients $\{Y_i\}$ to construct a low-variance estimator:
$$
g_t = \bar{Y} + \tfrac{1}{|S_t|}\sum_{i \in S_t} (F_i(W_t) - Y_i),
$$
where $\bar{Y}$ denotes the running mean.
This estimator asymptotically matches the unbiased gradient of the surrogate loss but with reduced stochastic noise.
The Jacobian $B(W)$ is learned either via \textsc{Probe} (periodic perturbations) or \textsc{Dither} (continuous randomization).
This VR variant is the one used in all main-text experiments and theoretical results in \Cref{sec:convergence}.

\begin{algorithm}[t]
\caption{\textbf{\jbase} (Non-Variance-Reduced):
Unified Jacobian Learning for Uniform QAT (Appendix)}
\label{alg:jacquant-unified-no-vr}
\small
\begin{algorithmic}[1]
\Require Pretrained $W$, dataset $\mathcal D$, quantizer $Q_\Delta$, step size $\eta$, batch size $b$, EMA $\beta\in(0,1)$;
\Require \textbf{mode} $\in\{\gbox{\textsc{Probe}},\,\gbox{\textsc{Dither}}\}$; (Probe) probe scale $\sigma$, interval $T_{\text{probe}}$
\State Initialize $B \leftarrow I$ (diagonal or per-group)
\For{$t=1,2,\dots$}
  \State Sample minibatch $S_t \subset \mathcal D$
  \If{\gbox{mode = DITHER}}
     \State Draw $r \sim \mathcal U[-\frac{\Delta}{2},\frac{\Delta}{2}]$, set $q \leftarrow Q_\Delta(W + r) - r$ \Comment{\gbox{Unbiased in mean}}
  \Else \Comment{\gbox{Probe}}
     \State $q \leftarrow Q_\Delta(W)$
  \EndIf
  \State \textit{// Per-sample grads and learned-Jacobian backprop}
  \For{$i \in S_t$}
     \State $v_i \leftarrow \nabla_q \ell(f(x_i;q), y_i)$
  \EndFor
  \State $g_t \leftarrow \frac{1}{|S_t|}\sum_{i\in S_t} B(W)\,v_i$
  \State \textit{// Any first-order optimizer (SGD/AdamW)}
  \State $W \leftarrow W - \eta\, g_t$
  \If{\gbox{mode = PROBE} \textbf{and} $t \bmod T_{\text{probe}} = 0$}
     \For{each group $g$}
        \State Draw $\delta_g \sim \mathcal N(0,\sigma^2 I)$,   $W_g^+ \leftarrow W_g + \delta_g$
        \State $\Delta q_g \leftarrow Q_\Delta(W_g^+) - Q_\Delta(W_g)$
        \State $B_g \leftarrow (1-\beta) B_g + \beta\,\mathrm{diag}\Big(\frac{\Delta q_g \odot \delta_g}{\delta_g \odot \delta_g + \epsilon}\Big)$
        \Comment{\gbox{Probe-based $B$ update}}
     \EndFor
  \ElsIf{\gbox{mode = DITHER}}
     \For{each group $g$}
        \State $B_g \leftarrow (1-\beta) B_g + \beta\,\phi_g(W,r)$
        \Comment{\gbox{EMA to $I$ in interiors; damp near clipping}}
     \EndFor
  \EndIf
\EndFor
\end{algorithmic}
\end{algorithm}

\textbf{\jbase (Non-VR Variant).}
For completeness, \Cref{alg:jacquant-unified-no-vr} provides the simpler non-variance-reduced version used for ablation studies.
This baseline removes the JacSketch control-variate mechanism and uses the plain stochastic gradient
$$
g_t = \tfrac{1}{|S_t|}\sum_{i\in S_t} F_i(W_t),
$$
while keeping the same learned Jacobian updates (\textsc{Probe} or \textsc{Dither}).
Although it has higher variance, this variant more directly reflects standard \ste-style QAT pipelines and is computationally lighter.
Empirically, it converges more slowly and exhibits larger oscillations near quantization boundaries—consistent with the variance-reduction analysis in \Cref{subsec:vr-framework,thm:nonconvex-generic}.

\begin{table}[!htbp]
\centering
\caption{Comparison between \algn{JacQuant-VR} and \algn{JacQuant-Base}.
\algn{VR} integrates variance-reduced control variates (\algn{SAGA}, \algn{SVRG}, \algn{SARAH}),
while \algn{Base} uses direct stochastic gradients.
Both support Probe- and Dither-based Jacobian updates.}
\label{tab:jacquant-comparison-variants}
\small
\setlength{\tabcolsep}{4pt}
\renewcommand{\arraystretch}{1.5}
\begin{tabular}{p{3.2cm}p{5.4cm}p{5.4cm}}
\toprule
\textbf{Feature} & \textbf{\jvr (Main)} & \textbf{\jbase (Ablation)} \\
\midrule
Gradient estimator & Control-variate (\algn{JacSketch}) & Plain stochastic gradient \\
Variance level & Low ($\sim$$\rho\sigma^2$, $\rho{<}1$) & High ($\sigma^2$) \\
Memory overhead & $\approx1.5\%$ (gradient table) & Negligible \\
Convergence rate & $\mathcal{O}(1/T)$ / linear under PL & Same asymptotic, larger constants \\
Stability near clipping & Strong (damped $B(W)$ + low variance) & Moderate (no variance damping) \\
Recommended use & Main experiments, high precision & Ablations, small models \\
\bottomrule
\end{tabular}
\end{table}

\textbf{Comparison.}
\Cref{tab:jacquant-comparison-variants} contrasts the two implementations.
\algn{JacQuant-VR} achieves lower variance and smoother training dynamics at the cost of a modest memory overhead for storing historical gradient tables (about $1.5\%$ of Adam states).
\algn{JacQuant-Base} is stateless and slightly faster per step, but less stable in ultra-low-bit regimes (e.g., \texttt{W2A16}).
Both share the same theoretical convergence guarantees when $B(W)\to J(W)$, but differ in variance constants.

Overall, the variance-reduced version is recommended for high-precision or large-scale training, while the base variant serves as a minimal, easy-to-integrate alternative.

\section{Detailed Experimental Setup and Baselines}\label{app:exp-setup}

\subsection{Experimental Setup}
\textbf{Models and data.}
We evaluate on three representative LLM architectures: \texttt{LLaMA3-1B}, \texttt{LLaMA3-3B} \citep{llama3}, and \texttt{Qwen3-1.7B} \citep{yang2025qwen3}. We also report a \texttt{LLaMA3-8B} W2A16 trajectory to probe scaling behavior beyond the 1--3B matched-ablation suite.
All models are trained on \texttt{FineWebEdu}~\citep{penedo2024fineweb} for pre-training QAT and evaluated on \texttt{WikiText-2}~\citep{merity2017pointer} for perplexity, plus eight standard zero-shot benchmarks including \texttt{ARC-easy/challenge}~\citep{clark2018think}, \texttt{BoolQ}~\citep{clark2019boolq}, \texttt{PIQA}~\citep{bisk2020piqa}, \texttt{SIQA}~\citep{sap2019socialiqa}, \texttt{HellaSwag}~\citep{zellers2019hellaswag}, \texttt{OBQA}~\citep{mihaylov2018can}, and \texttt{WinoGrande}~\citep{sakaguchi2021winogrande}.

\textbf{Quantization configuration.}
We apply symmetric uniform quantization with per-channel scaling to all linear layers (attention Q/K/V projections, MLP up/down projections, output head).
Bit-widths range from 1 to 2 bits for weights (\texttt{W1}, \texttt{W1.58}, \texttt{W2}) and 8 to 16 bits for activations (\texttt{A8}, \texttt{A16}).
We focus primarily on ultra-low-bit regimes (\texttt{W$\leq$2}) where \ste-based methods struggle most. Each quantization group contains 32-256 parameters; we use 128 as the default.

\textbf{Training Details.}
All methods use identical forward quantization $Q(\cdot)$; only the \emph{backward gradient rule} differs. We train with AdamW ($\beta_1=0.9$, $\beta_2=0.999$, $\varepsilon=10^{-8}$), learning rate $\eta=2\times10^{-5}$, constant schedule with 1k warmup steps, weight decay $0$, batch size 8 per device, and gradient accumulation 1.
Training runs for 240k steps with evaluation every 5k steps. For \texttt{LLaMA3-8B}, we report intermediate checkpoints at 30K/60K/90K due to the substantially higher end-to-end QAT cost. For \method, we update the surrogate Jacobian $B(W)$ every 100 steps using EMA with $\beta=0.9$ by default.

We perform a systematic grid search to determine the optimal hyperparameters for \method. The search spans both the probe-based (\probe) and dither-based (\dither) variants, exploring probe noise levels $\sigma \in \{10^{-4}, 10^{-3}, 10^{-2}\}$, EMA coefficients $\beta \in \{0.8, 0.9, 0.95\}$, update frequencies $f_{\text{update}} \in \{50, 100, 200\}$, and group sizes $G \in \{32, 64, 128, 256\}$. We also evaluate the impact of optional JacSketch control variates for variance reduction.
The final results are reported using the best-performing hyperparameter settings selected by validation perplexity on the \texttt{WikiText-2} dataset. This grid is a development-time sweep rather than an added deployment requirement; after selecting defaults, the sensitivity study in \Cref{tab:ablation-gs-uf} shows that performance varies mildly over a broad range of group sizes and refresh intervals.

\textbf{Computational Overhead.}
The additional memory cost of \method for storing $B(W)$ is $\mathcal{O}(G) \approx 15$--$30$ MB for typical LLMs, which is less than $1.5\%$ of Adam optimizer states.
Jacobian updates add negligible runtime overhead ($<2\%$) since they occur infrequently (every 50--200 steps).The runtime study in \Cref{tab:runtime_w2a16,tab:runtime_llama3_8b} reports measured overhead explicitly for both \texttt{LLaMA3-1B} W2A16 and \texttt{LLaMA3-8B} W2A16.
For comparison, true gradient computation via finite differences (used only for validation) adds $10$--$20\%$ overhead but is not needed during standard training.

\textbf{Baselines}
We compare against representative QAT baselines: (i) \algn{STE}~\citep{bengio2013estimating}, the standard straight-through estimator with identity Jacobian; (ii) \algn{ParetoQ}~\citep{liu2025paretoq} and \algn{WinQ}~\citep{winq2025}, which optimize Pareto-balanced or noise-regularized quantization under the same \ste assumption; and (iii) our proposed \probe and \dither, which learn surrogate Jacobians via probing and dithering (\Cref{alg:jacquant-unified}).

\textit{Unless explicitly marked as \method-\textsc{Base} in \Cref{app:jacquant-variants}, all reported \method results use the variance-reduced implementation; the base/non-VR variant is included only to isolate the effect of variance reduction.} All baselines use identical model architectures, training data, optimization settings, and quantization configurations. All experiments are implemented in PyTorch 2.3 with custom CUDA kernels for quantization on 8 NVIDIA A100 (80GB) GPUs.

\section{Zero-Shot Generalization: Complementary Results}\label{app:zero_shot}
\Cref{tab:main_results}-\ref{tab:llama3_8b_results} reports the mean zero-shot generalization performance over eight standard QA reasoning tasks. 
We provide the full set of complementary results for LLaMA3-1B  (\Cref{tab:llama1b_a16}), 
LLaMA3-3B (\Cref{tab:llama_3b}), Qwen-1.7B (\Cref{tab:qwen}) and LLaMA3-8B (\Cref{tab:llama8b}).

\begin{table*}[!htbp]
\centering
\caption{%
\textbf{Zero-shot accuracy (\%) on downstream reasoning tasks. \texttt{LLaMA3-1B}.}
Models are trained with QAT on \texttt{FineWebEdu} and directly evaluated on reasoning benchmarks without fine-tuning.
Higher is better. Best per column in \textbf{bold}.
}
\label{tab:llama1b_a16}
\small
\setlength{\tabcolsep}{3.5pt}
\begin{tabular}{l@{\hskip 6pt}ccccccccc|@{\hskip 6pt}c}
\toprule
Method & \textbf{Wiki2 ($\downarrow$)} & ARC-e & ARC-c & BoolQ & PIQA & SIQA & HellaS. & OBQA & Wino. & \textbf{Avg. Acc. ($\uparrow$)}\\
\midrule
FP16 Baseline & 9.6 & 64.8 & 42.5& 64.8 & 74.8 & 44.8 & 64.4 & 50.2 & 61.5& 58.5 \\ \midrule\midrule
\multicolumn{10}{l}{\texttt{LLaMA3-1B, W1A16}} \\
\midrule
\algn{RTN} & 4.2e8 & 25.0 & 22.5 & 37.6 & 49.5 & 32.9 & 25.0 & 27.1 & 49.6 & 33.7 \\
GPTQ & 3.3e8 & 26.9 & 21.7 & 37.6 & 51.8 & 33.5 & 25.5 & 14.8 & 49.7 & 32.7\\
SpinQuant & 2.4e8 & 25.0 & 22.5 & 37.6 & 49.5 & 32.9 & 25.0 & 27.1 & 49.6 & 33.7 \\ \midrule
\algn{ParetoQ} & 17.9 & 58.7&	39.8&	62.6&	67.4&	42.4&	46.9&	43.9&	53.5&	51.9\\
\ \ $+$\method & \textbf{17.6} & 60.6 & 38.9 & 64.1 & 67.2 & 42.6 & 47.8 & 44.1 & 53.9 & \textbf{52.4}\\ \midrule
\algn{WinQ} & 15.3 & 59.7 & 37.0 & 61.3 & 69.5 & 42.6 & 49.8 & 46.1 & 54.4 & 52.6\\
\ \ $+$\method & \textbf{15.1} & 61.1 & 37.9 & 63.1 & 70.2 & 43.2 & 50.9 & 44.9 & 56.3 & \textbf{53.5}\\ 
\midrule\midrule
\multicolumn{10}{l}{\texttt{LLaMA3-1B, W1.58A16}} \\
\midrule
RTN & 1.8e6 & 24.5 & 22.3 & 62.4 & 52.7 & 33.4 & 25.4 & 18.4 & 50.2 & 36.2 \\ 
GPTQ & 7.5e4 & 24.8 & 22.2 & 38.0 & 52.2 & 32.5 & 25.4 & 17.0 & 49.4 & 32.7  \\
SpinQuant & 5.8e3 & 25.3 & 22.5 & 37.6 & 51.6 & 33.3 & 25.3 & 17.6 & 48.5 & 32.7 \\ \midrule
\algn{ParetoQ} & 14.8 & 61.4&	39.7&	63.1&	71.0&	42.6&	53.1&	47.5&	55.7&	54.3 \\
\ \ $+$\method & \textbf{13.9} & 64.2&	38.2&	62.2&	71.6&	44.4&	55.0&	47.7&	59.1&	\textbf{55.3}\\ \midrule
\algn{WinQ} & 12.9 & 65.0 & 39.7 & 62.1 & 72.9 & 44.3 & 56.2 & 47.1 & 57.4 & 55.6\\
\ \ $+$\method & \textbf{12.3} & 66.1 & 39.2 & 63.4 & 73.6 & 44.8 & 56.9 & 47.5 & 57.1 & \textbf{56.1}\\ 
\midrule\midrule
\multicolumn{10}{l}{\texttt{LLaMA3-1B, W2A16}} \\ \midrule
RTN & 1.5e6 & 24.5 & 23.1 & 62.4 & 52.3 & 33.6 & 25.4 & 17.6 & 50.3 & 36.2 \\ 
GPTQ & 3.8e4 & 26.9 & 21.7 & 37.6 & 51.8 & 33.5 & 25.5 & 14.8 & 49.7 & 32.7  \\
SpinQuant & 3.8e2 & 31.0 & 20.1 & 45.5 & 54.6 & 33.8 & 26.7 & 16.2 & 50.9 & 34.9 \\ \midrule
\algn{ParetoQ} & 13.5 & 63.1&	41.9&	62.5&	72.8&	43.6&	55.2&	49.2&	57.3&	55.7\\
\ \ $+$\method & \textbf{12.3} & 66.1 & 42.0 & 62.1 & 72.8 & 43.3 & 57.4 & 49.8 & 59.1 & \textbf{56.6}\\ \midrule
\algn{WinQ} & 11.9 & 65.0 & 42.5 & 62.5 & 73.8 & 43.2 & 58.4 & 48.1 & 59.1 & 56.6\\
\ \ $+$\method & \textbf{11.8} & 66.4 & 39.8 & 64.7 & 73.3 & 43.8 & 59.2 & 47.8 & 59.4 & \textbf{56.8}\\ 
\bottomrule
\end{tabular}
\vspace{-0.5em}
\end{table*}

\begin{table*}[!htbp]
\centering
\caption{%
\textbf{Zero-shot accuracy (\%) on downstream reasoning tasks. \texttt{LLaMA3-3B}.}
Models are trained with QAT on \texttt{FineWebEdu} and directly evaluated on reasoning benchmarks without fine-tuning.
Higher is better. Best per column in \textbf{bold}.
}
\label{tab:llama_3b}
\small
\setlength{\tabcolsep}{3.5pt}
\begin{tabular}{l@{\hskip 6pt}ccccccccc|@{\hskip 6pt}c}
\toprule
Method & \textbf{Wiki2 ($\downarrow$)} & ARC-e & ARC-c & BoolQ & PIQA & SIQA & HellaS. & OBQA & Wino. & \textbf{Avg. Acc. ($\uparrow$)}\\
\midrule
FP16 Baseline & 7.7 & 72.6 & 50.7 & 74.6 & 78.2 & 48.5 & 74.3 & 53.7 & 69.2 & 65.2 \\ \midrule\midrule
\multicolumn{10}{l}{\texttt{LLaMA3-3B, W1A8}} \\
\midrule
RTN & 7.3e7 & 25.0 & 22.5 & 37.6 & 49.5 & 32.9 & 25.0 & 27.1 & 49.6 & 33.7\\ 
GPTQ & 5.9e7  & 26.6 & 21.9 & 37.6 & 52.5 & 33.7 & 25.6 & 15.0 & 49.4 & 32.8  \\
SpinQuant & 4.5e7 & 27.0 & 22.0 & 37.6 & 52.0 & 33.4 & 25.5 & 14.5 & 50.2 & 32.8 \\ \midrule
ParetoQ & 15.7 & 63.1 & 39.1 & 63.0 & 70.9 & 42.6 & 50.6 & 46.7 & 56.5 & 54.1 \\ 
\ \ $+$\method & \textbf{15.2} & 64.3& 38.9&	60.6&	72.2&	43.2&	50.7&	49.8&	59.1&	\textbf{54.9}	\\
\midrule\midrule
\multicolumn{10}{l}{\texttt{LLaMA3-3B, W1.58A8}} \\
\midrule
RTN & 7.9e5 & 24.7 & 23.4 & 37.6 & 52.9 & 33.8 & 25.5 & 17.4 & 50.3 & 33.2 \\ 
GPTQ & 2.7e5 & 25.3 & 22.9 & 37.6 & 53.7 & 34.4 & 25.3 & 15.6 & 49.5 & 33.0  \\
SpinQuant & 3.1e3 & 26.6 & 23.4 & 39.1 & 52.7 & 34.7 & 25.3 & 16.4 & 48.5 & 33.3 \\ \midrule
\algn{ParetoQ} & 13.1 & 67.9 & 42.0 & 53.9 & 72.0 & 43.9 & 58.2 & 49.6 & 60.2 & 56.0 \\ 
\ \ $+$\method & \textbf{12.7} & 67.3&	42.0&	57.5&	73.4&	43.8&	59.1&	51.8&	61.3&	\textbf{57.1}\\ 
\bottomrule
\end{tabular}
\vspace{-0.5em}
\end{table*}

\begin{table*}[!htbp]
\centering
\caption{%
\textbf{Zero-shot accuracy (\%) on downstream reasoning tasks. \texttt{Qwen-1.7B}.}
Models are trained with QAT on \texttt{FineWebEdu} and directly evaluated on reasoning benchmarks without fine-tuning.
Higher is better. Best per column in \textbf{bold}.
}
\label{tab:qwen}
\small
\setlength{\tabcolsep}{3.5pt}
\begin{tabular}{l@{\hskip 5pt}ccccccccc|@{\hskip 5pt}c}
\toprule
Method & \textbf{Wiki2 ($\downarrow$)} & ARC-e & ARC-c & BoolQ & PIQA & SIQA & HellaS. & OBQA & Wino. & \textbf{Avg. Acc. ($\uparrow$)}\\
\midrule
FP16 Baseline & 16.2 & 68.9 & 41.0 & 78.9 & 71.7 & 45.1 & 59.6 & 38.5 & 61.6 & 58.2 \\ \midrule\midrule
\multicolumn{10}{l}{\texttt{Qwen-1.7B, W1A8}} \\ \midrule
\algn{ParetoQ} & 46.5 & 42.0 & 26.2 & 61.4 & 59.8 & 40.0 & 29.1 & 28.9 & 50.8 & 42.3\\
\ \ $+$\method & \textbf{45.3} & 42.5 & 26.5 & 61.8 & 59.6 & 40.6 & 29.5 & 29.4 & 51.7 & \textbf{42.7}\\ 
\midrule\midrule
\multicolumn{10}{l}{\texttt{Qwen-1.7B, W2A8}} \\\midrule
\algn{ParetoQ} & 22.2 & 52.9 & 32.6 & 62.4 & 65.1 & 42.7 & 41.2 & 32.8 & 53.1 & 47.9 \\
\ \ $+$\method & \textbf{21.5} & 54.1 & 31.1 & 64.4 & 64.1 & 43.5 & 40.3 & 34.1 & 52.8 & \textbf{48.1}\\ 
\bottomrule
\end{tabular}
\vspace{-0.5em}
\end{table*}

\begin{table*}[!htbp]
\centering
\caption{%
\textbf{Zero-shot accuracy (\%) on downstream reasoning tasks. \texttt{LLaMA3-8B} with \texttt{W2A16}.}
Models are trained with QAT on \texttt{FineWebEdu} and directly evaluated on reasoning benchmarks without fine-tuning.
Higher is better. Best per column in \textbf{bold}.
}
\label{tab:llama8b}
\small
\setlength{\tabcolsep}{3.5pt}
\begin{tabular}{l@{\hskip 5pt}ccccccccc|@{\hskip 5pt}c}
\toprule
Method & \textbf{Wiki2 ($\downarrow$)} & ARC-e & ARC-c & BoolQ & PIQA & SIQA & HellaS. & OBQA & Wino. & \textbf{Avg. Acc. ($\uparrow$)}\\
\midrule
\multicolumn{10}{l}{\texttt{30K}} \\ \midrule
\algn{ParetoQ} & 11.78 & 70.79 & 42.54 & 67.37 & 74.80 & 45.41 & 66.12 & 37.89 & 63.39 & 58.54 \\
\ \ $+$\method & \textbf{11.69} & 69.65 & 44.95 & 67.55 & 74.88 & 44.86 & 66.65 & 39.94 & 64.32 & \textbf{59.10} \\ 
\midrule\midrule
\multicolumn{10}{l}{\texttt{60K}} \\ \midrule
\algn{ParetoQ} & 10.43 & 72.50 & 46.58 & 69.30 & 75.56 & 46.88 & 69.09 & 40.72 & 64.70 & 60.67\\
\ \ $+$\method & \textbf{10.32} & 72.74 & 46.21 & 70.73 & 75.42 & 46.92 & 68.91 & 41.12 & 65.80 & \textbf{60.98} \\ 
\midrule\midrule
\multicolumn{10}{l}{\texttt{90K}} \\\midrule
\algn{ParetoQ} & 9.79 & 75.60 & 48.36 & 66.91 & 76.79 & 47.13 & 69.36 & 40.03 & 67.07 & 61.41  \\
\ \ $+$\method & \textbf{9.71} & 75.48 & 49.42 & 70.31 & 76.93 & 47.53 & 69.44 & 41.31 & 65.88 & \textbf{62.04} \\ 
\bottomrule
\end{tabular}
\vspace{-0.5em}
\end{table*}

\newpage
\section{Runtime Overhead}\label{sec:runtime_w2a16}
A key practical goal of \method is to improve low-bit QAT stability \emph{without} sacrificing training throughput.
Table~\ref{tab:runtime_w2a16} reports end-to-end training throughput (steps/second) for \texttt{LLaMA3-1B} W2A16 when augmenting two strong QAT
baselines (\algn{ParetoQ} and \algn{WinQ}) with our Jacobian-learning backpropagation (\probe / \dither).

\paragraph{Group size vs.\ throughput.}
We vary the Jacobian \emph{group size} $g_s$, i.e., the number of weight elements that share a single scale/Jacobian entry in the
block-diagonal surrogate $B(W)$.
Larger $g_s$ reduces the number of groups and thus decreases both (i) the memory traffic for reading/writing $B(W)$ and
(ii) the amount of work in groupwise Jacobian updates (\texttt{ProbeUpdate}/\texttt{DitherEMA}), so it is expected to yield \emph{higher} steps/second. This trend is visible in several settings; e.g., for \algn{WinQ} with \probe, throughput increases from $1.7223$ ($g_s=8$)
to $1.7392$ ($g_s=32$) to $1.7447$ ($g_s=128$).
Despite the above expectation, the measured throughput is not strictly monotonic in all cases.
For example, under \algn{WinQ}+\dither, $g_s{=}32$ yields $1.7411$ steps/s while $g_s{=}128$ yields $1.7308$ steps/s.
This difference is small (about $0.6\%$) and is well within typical run-to-run and measurement variability for large-scale
training, where throughput can be affected by CUDA kernel scheduling, cache effects, data-loader jitter, and asynchronous
execution/overlap.
Crucially, the Jacobian-learning component is only a \emph{small} fraction of the total transformer forward/backward cost,
so these secondary system effects can mask the otherwise expected scaling with $g_s$ when differences are at the sub-1\%
level.

\paragraph{Minimal overhead in practice.}
Across both baselines and both Jacobian estimators, \method maintains throughput close to STE training.
Relative to the STE baseline, the slowdown ranges from roughly \emph{1--5\%} in this {\texttt{LLaMA3-1B}} W2A16 setting, while delivering the
accuracy/stability benefits reported elsewhere.
In particular, with the recommended larger group sizes (e.g., $g_s{=}128$), the overhead is consistently small (about
$1.6\%$ for \algn{WinQ}+\probe and $3.5\%$ for \algn{ParetoQ}+\probe), confirming that \method is a practical drop-in
replacement: the forward quantizer is unchanged, and the additional cost of learning/applying $B(W)$ is amortized by its
lightweight block structure and infrequent updates.

\begin{table}[!htbp]
    \centering
    \small
    \setlength{\tabcolsep}{5pt}
    \caption{\textbf{Training throughput} (steps/s; higher is better) for {\texttt{LLaMA3-1B}} W2A16 QAT.
    $g_s$ is the number of weights that share one scale / one (block-)Jacobian entry.
    Parentheses report relative throughput change vs.\ STE.}
    \label{tab:runtime_w2a16}
    \begin{tabular}{l|c|ccc|ccc}
    \toprule
    Base & \ste &
    \multicolumn{3}{c|}{\probe ($g_s$)} &
    \multicolumn{3}{c}{\dither ($g_s$)} \\
    \cmidrule(lr){3-5}\cmidrule(lr){6-8}
    & & 8 & 32 & 128 & 8 & 32 & 128 \\
    \midrule
    \algn{ParetoQ} \cite{liu2025paretoq} & 1.8074 &
    \begin{tabular}[c]{@{}c@{}}1.7391\\{\scriptsize(-3.8\%)}\end{tabular} &
    \begin{tabular}[c]{@{}c@{}}1.7135\\{\scriptsize(-5.2\%)}\end{tabular} &
    \begin{tabular}[c]{@{}c@{}}1.7435\\{\scriptsize(-3.5\%)}\end{tabular} &
    \begin{tabular}[c]{@{}c@{}}1.7194\\{\scriptsize(-4.9\%)}\end{tabular} &
    \begin{tabular}[c]{@{}c@{}}1.7347\\{\scriptsize(-4.0\%)}\end{tabular} &
    \begin{tabular}[c]{@{}c@{}}1.7334\\{\scriptsize(-4.1\%)}\end{tabular} \\
    \algn{WinQ} \cite{winq2025} & 1.7738 &
    \begin{tabular}[c]{@{}c@{}}1.7223\\{\scriptsize(-2.9\%)}\end{tabular} &
    \begin{tabular}[c]{@{}c@{}}1.7392\\{\scriptsize(-2.0\%)}\end{tabular} &
    \begin{tabular}[c]{@{}c@{}}1.7447\\{\scriptsize(-1.6\%)}\end{tabular} &
    \begin{tabular}[c]{@{}c@{}}1.7229\\{\scriptsize(-2.9\%)}\end{tabular} &
    \begin{tabular}[c]{@{}c@{}}1.7411\\{\scriptsize(-1.8\%)}\end{tabular} &
    \begin{tabular}[c]{@{}c@{}}1.7308\\{\scriptsize(-2.4\%)}\end{tabular} \\
    \bottomrule
    \end{tabular}
\end{table}

\paragraph{Runtime on \texttt{LLaMA3-8B}.}
\Cref{tab:runtime_llama3_8b} reports the corresponding raw step-time analysis for \texttt{LLaMA3-8B} under the \algn{ParetoQ} W2A16 pipeline. The STE baseline takes 1297.01 ms/step (0.7710 steps/s). Practical Probe settings add only 10.17--10.18 ms/step for $g_s\in\{32,128\}$, corresponding to a $0.77$--$0.78\%$ throughput change; Dither remains within $1.62\%$ across all tested group sizes, with the small apparent speedup at $g_s=32$ best interpreted as measurement noise rather than a computational advantage.

\begin{table}[!htbp]
    \centering
    \small
    \setlength{\tabcolsep}{7pt}
    \caption{\textbf{\texttt{LLaMA3-8B} runtime under \algn{ParetoQ} W2A16.} Step time is wall-clock milliseconds per training step; steps/s is recomputed from step time. Parentheses report relative throughput change vs. the STE baseline.}
    \label{tab:runtime_llama3_8b}
    {
    \begin{tabular}{lccc}
    \toprule
    Configuration & Step time (ms) & Steps/s & Relative to STE \\
    \midrule
    \algn{ParetoQ} STE baseline & 1297.01 & 0.7710 & -- \\
    $+$\probe, $g_s=8$   & 1335.81 & 0.7486 & -2.91\% \\
    $+$\probe, $g_s=32$  & 1307.18 & 0.7650 & -0.78\% \\
    $+$\probe, $g_s=128$ & 1306.99 & 0.7651 & -0.77\% \\
    $+$\dither, $g_s=8$   & 1300.75 & 0.7688 & -0.29\% \\
    $+$\dither, $g_s=32$  & 1293.20 & 0.7734 & +0.31\% \\
    $+$\dither, $g_s=128$ & 1318.38 & 0.7585 & -1.62\% \\
    \bottomrule
    \end{tabular}}
\end{table}

\section{Ablation Studies}
\label{app:ablations}
\paragraph{Setup.}
We ablate key hyperparameters of \method on W2A16 training using \algn{ParetoQ}+\method as the base pipeline. We report validation perplexity on WikiText-2 (Wiki2; lower is better) and zero-shot accuracy on eight reasoning benchmarks (higher is better), along with the mean accuracy (Avg. Acc.). Unless otherwise stated, we use the same optimizer and training recipe as in the main experiments. For the Probe-based Jacobian learner, we fix the probe noise scale to $\sigma=10^{-4}$ and set the refresh probability to $p=10^{-3}$.

\subsection{Effect of Jacobian Group Size and Update Frequency}
\label{app:ablation-gs-uf}
\method uses a lightweight diagonal/block-diagonal surrogate Jacobian $B(W)$. Two practical knobs govern its capacity and stability:

\begin{itemize}
    \item \textbf{Group size (gs):} the number of weight elements that share a single (block-)Jacobian entry. Smaller \texttt{gs} increases expressivity (finer-grained sensitivity modeling) but can raise estimation noise and overhead.
    \item \textbf{Update frequency (uf):} how often we refresh the Jacobian surrogate (in steps). More frequent updates can track drift faster, but may inject higher-frequency noise into the backward pass.
\end{itemize}

\begin{table}[!htbp]
\centering
\small
\setlength{\tabcolsep}{3.5pt}
\caption{Sensitivity to Jacobian \emph{group size} (gs) and Jacobian \emph{update frequency} (uf) for \algn{ParetoQ}+\method under W2A16.}
\label{tab:ablation-gs-uf}
\begin{tabular}{cc|c|ccccccccc}
\toprule
\textbf{gs} & \textbf{uf} & \textbf{Wiki2} ($\downarrow$) & ARC-e & ARC-c & BoolQ & PIQA & SIQA & HellaS. & OBQA & Wino. & \textbf{Avg.\ Acc.} ($\uparrow$) \\
\midrule
8   & 50  & 12.6 & 65.2 & 41.0 & 63.3 & 73.9 & 44.8 & 56.8 & 50.4 & 57.3 & 56.6 \\
8   & 100 & 12.6 & 64.9 & 40.0 & 60.2 & 73.2 & 43.8 & 56.8 & 47.7 & 58.9 & 55.7 \\
8   & 200 & 12.5 & 65.5 & 38.4 & 61.0 & 73.1 & 43.1 & 57.2 & 49.2 & 58.0 & 55.7 \\
\midrule
32  & 50  & 12.5 & 64.2 & 42.3 & 61.2 & 72.1 & 43.6 & 57.4 & 49.4 & 56.3 & 55.8 \\
32  & 100 & 12.5 & 64.8 & 40.5 & 61.9 & 72.6 & 43.9 & 56.8 & 51.4 & 57.8 & 56.2 \\
32  & 200 & \textbf{12.3} & 66.0 & 42.0 & 62.1 & 72.8 & 43.3 & 57.4 & 49.8 & 59.1 & \textbf{56.6} \\
\midrule
128 & 50  & 12.5 & 65.3 & 41.2 & 62.1 & 73.1 & 42.7 & 57.3 & 49.4 & 57.6 & 56.1 \\
128 & 100 & 12.3 & 64.4 & 41.7 & 62.6 & 73.1 & 43.3 & 57.0 & 50.2 & 58.5 & 56.3 \\
128 & 200 & 12.5 & 65.1 & 42.8 & 62.5 & 74.0 & 43.4 & 57.4 & 48.4 & 58.4 & 56.5 \\
\bottomrule
\end{tabular}
\end{table}

\paragraph{Findings.}
\textbf{(1) Moderate group sizes are best in this regime.}
Although smaller \texttt{gs} increases representational capacity (fewer weights share one Jacobian scalar), we find the best overall setting at \texttt{gs=32}. A plausible explanation is that very fine-grained grouping (\texttt{gs=8}) makes the probe-based estimation problem noisier and may require either (i) more training steps, or (ii) a more careful learning-rate / refresh schedule to realize its potential. We leave a more systematic scheduler/horizon study for future work.

\noindent\textbf{(2) Jacobian updates should not be too frequent.}
For \texttt{gs=32} and \texttt{gs=128}, increasing \texttt{uf} consistently improves performance, and \texttt{uf=200} yields the strongest overall results. This supports the intuition that overly frequent updates can inject high-frequency noise into $B(W)$ and destabilize the surrogate gradients, whereas
less frequent refreshes amortize the estimation and allow the optimizer to exploit a more stable backward model.

\subsection{Probe vs. Dither Jacobian Learners}
\label{app:ablation-probe-vs-dither}
We next compare the two Jacobian-learning mechanisms described in \S\ref{sec:method}:
(i) \probe, which estimates local slopes via Gaussian probes, and
(ii) \dither, which leverages subtractive dithering.
Both variants are evaluated using the same configuration (\texttt{gs=32}, \texttt{uf=200}).

\begin{table}[!htbp]
\centering
\small
\setlength{\tabcolsep}{3.5pt}
\caption{\probe vs.\ \dither under \texttt{gs=32}, \texttt{uf=200} (W2A16).
Probe yields stronger WikiText-2 perplexity and higher average accuracy in our current implementation.}
\label{tab:ablation-probe-dither}
\begin{tabular}{l|c|ccccccccc}
\toprule
\textbf{Method} & \textbf{Wiki2} ($\downarrow$) & ARC-e & ARC-c & BoolQ & PIQA & SIQA & HellaS. & OBQA & Wino. & \textbf{Avg.\ Acc.} ($\uparrow$) \\
\midrule
\dither & 12.5 & 64.2 & 42.3 & 61.2 & 72.1 & 43.6 & 57.4 & 49.4 & 56.3 & 55.8 \\
\probe  & \textbf{12.3} & 66.0 & 42.0 & 62.1 & 72.8 & 43.3 & 57.4 & 49.8 & 59.1 & \textbf{56.6} \\
\bottomrule
\end{tabular}
\end{table}

\paragraph{Takeaway.}
\probe outperforms \dither in this configuration, improving both WikiText-2 perplexity and average zero-shot accuracy. Consequently, we use \probe as the default Jacobian learner in the main experiments. Improving the dither estimator (e.g., via refined statistics or tighter coupling to clipping dynamics) is an interesting direction for future work.

\section{Missing Proofs}\label{app:theory}
\subsection{Proof of \Cref{lem:probe-consistency}}
\begin{proof}
\textbf{Step 0 (Population target and orthogonality).}
On a fixed window, define the \emph{population} least-squares target $J_g(W)$ as
$$
J_g(W)\ \in\ \arg\min_{B\in\mathcal{B}_g}\ \E\big[\|\Delta q_g - B \delta_g\|^2\big],
$$
where $\mathcal{B}_g$ denotes the admissible class (diagonal or block-diagonal matrices on group $g$).
The first-order (normal) equations at the population optimum yield
\begin{equation}
\label{eq:pop-normal}
\E \big[\delta_g \delta_g^\top\big] J_g(W)^\top  =  \E \big[\delta_g \Delta q_g^\top\big],
\qquad\text{equivalently}\qquad
\E \big[\Delta q_g \delta_g^\top\big]  =  J_g(W) \E \big[\delta_g\delta_g^\top\big].
\end{equation}
Define the residual $\varepsilon_g \eqdef \Delta q_g - J_g(W) \delta_g$. Then \eqref{eq:pop-normal}
is equivalent to the \emph{moment orthogonality}
\begin{equation}
\label{eq:orthogonality}
\E \big[\varepsilon_g \delta_g^\top\big]  =  0.
\end{equation}

\smallskip
\textbf{Step 1 (Finite-sample OLS form).}
Stack the $m$ samples for group $g$ as a design matrix
$X\in\R^{m\times d_g}$ with rows $\delta_{g,k}^\top$,
and a response matrix $Y\in\R^{m\times d_g}$ with rows $\Delta q_{g,k}^\top$.
Let $\widehat B_g$ be the (restricted) least-squares minimizer over $\mathcal{B}_g$:
$$
\widehat B_g\ \in\ \arg\min_{B\in\mathcal{B}_g}\ \|Y - X B^\top\|_F^2.
$$
If we first analyze the \emph{unrestricted} OLS solution
$\widetilde B_g^\top = (X^\top X)^{-1} X^\top Y$
and then project onto $\mathcal{B}_g$, the projection can only reduce the error
in operator norm (orthogonal projection is a contraction).
Thus it suffices to bound $\|\widetilde B_g - J_g(W)\|$.

Using $Y = X J_g(W)^\top + E$ with $E\in\R^{m\times d_g}$ whose rows are $\varepsilon_{g,k}^\top$,
$$
\widetilde B_g^\top - J_g(W)^\top
= (X^\top X)^{-1} X^\top E
\quad\Longrightarrow\quad
\|\widetilde B_g - J_g(W)\|
\ \le\ \big\|(X^\top X/m)^{-1}\big\| \cdot \big\|X^\top E/m\big\|.
$$
We will bound the two factors separately with high probability.

\smallskip
\textbf{Step 2 (Gram matrix concentration).}
Since $\delta_{g,k}\sim\mathcal N(0,\sigma^2 I_{d_g})$ i.i.d.,
by standard Wishart concentration (e.g., matrix Chernoff for Gaussian design),
there exist universal constants $c_1,c_2>0$ such that, with probability at least $1-\delta/2$,
\begin{equation}
\label{eq:gram-conc}
\big\|X^\top X/m - \sigma^2 I_{d_g}\big\| \ \le\ c_1 \sigma^2 \sqrt{\frac{d_g+\log(2/\delta)}{m}}
\ \eqdef\ \sigma^2  \eta_m,
\end{equation}
provided $m\gtrsim d_g+\log(1/\delta)$.
If $\eta_m\le 1/2$, then
$\lambda_{\min}(X^\top X/m)\ge \sigma^2(1-\eta_m)\ge \sigma^2/2$, hence
\begin{equation}
\label{eq:gram-inv}
\big\|(X^\top X/m)^{-1}\big\| \ \le\ \frac{1}{\sigma^2(1-\eta_m)} \ \le\ \frac{2}{\sigma^2}.
\end{equation}

\smallskip
\textbf{Step 3 (Cross term concentration).}
Consider the sum of independent random matrices
\(
Z_k \eqdef \delta_{g,k} \varepsilon_{g,k}^\top \in \R^{d_g\times d_g}.
\)
By \eqref{eq:orthogonality}, $\E[Z_k]=0$.
Assume the residuals $\varepsilon_{g,k}$ are conditionally sub-Gaussian with
$\psi_2$-norm bounded by a constant $K$ (this holds if $\Delta q_g$ is bounded/Lipschitz in the window or by standard dithering arguments);
$\delta_{g,k}$ are Gaussian with scale $\sigma$.
Then each $Z_k$ is sub-exponential in operator norm, with Orlicz $\psi_1$-norm bounded by
$\|Z_k\|_{\psi_1}\le c \sigma K$ for a universal $c$ (product of sub-Gaussians).
By the matrix Bernstein inequality for sums of independent, mean-zero, sub-exponential matrices,
there exist constants $c_3,c_4>0$ such that, with probability at least $1-\delta/2$,
\begin{equation}
\label{eq:cross-conc}
\big\| X^\top E/m \big\|
\ =\ \Big\| \frac{1}{m}\sum_{k=1}^m Z_k \Big\|
\ \le\ c_3 \sigma K \sqrt{\frac{d_g+\log(2/\delta)}{m}}.
\end{equation}

\smallskip
\textbf{Step 4 (Combine).}
On the intersection of the events in \eqref{eq:gram-inv} and \eqref{eq:cross-conc}
(which holds with probability at least $1-\delta$ by a union bound), we have
$$
\|\widetilde B_g - J_g(W)\|
\ \le\ \frac{2}{\sigma^2}\ \cdot\ c_3 \sigma K \sqrt{\frac{d_g+\log(2/\delta)}{m}}
\ \le\ C \frac{K}{\sigma} \sqrt{\frac{d_g+\log(1/\delta)}{m}},
$$
for $C=2c_3$ up to absorbing constants.
Finally, projecting $\widetilde B_g$ onto the admissible class $\mathcal{B}_g$
(diagonal or block-diagonal) is a non-expansive operation in operator norm,
so the same bound holds for $\widehat B_g$.

\smallskip
\textbf{Step 5 (Constant dependence).}
All constants $c_i$ depend only on universal constants in Gaussian/Wishart concentration
and on the sub-Gaussian parameter $K$ of the residual $\varepsilon_g$;
they do not depend on $m,d_g,\sigma$ beyond the explicit factors shown.
Setting $C\propto K$ yields the stated claim.
\end{proof}

\subsection{Proof of \Cref{prop:sa-tracking}}
\begin{proof}
\textbf{Notation and setup.}
Write $J_t \equiv J(W_t)$ for brevity.
On each (code-preserving) window, the probe-LS estimator $\widehat B_t$ targets $J_t$.
Let the admissible class be diagonal or block-diagonal per group; projection onto this class
is non-expansive in operator norm, so it suffices to work with $\widehat B_t$ directly.

\smallskip
\textbf{Step 1: One-step error decomposition.}
We have the affine recursion
$$
B_{t+1}-J_t  =  (1-\beta_t) (B_t-J_t)  +  \beta_t (\widehat B_t - J_t).
$$
Taking conditional expectation w.r.t.\ $\F_t$ (the sigma-field up to time $t$) and using Jensen,
\begin{equation}
\label{eq:cond-decomp}
\E \left[\|B_{t+1}-J_t\| \middle| \F_t\right]
 \le  (1-\beta_t) \|B_t-J_t\|  +  \beta_t \nu_t,
\end{equation}
where we set $\nu_t  \eqdef  \E \left[\|\widehat B_t - J_t\| \middle| \F_t\right]$.

\smallskip
\textbf{Step 2: Bounding the estimation error $\nu_t$.}
Uniform persistent excitation of probes implies a uniformly bounded, well-conditioned
Gram matrix on each window: there exists $\lambda>0$ such that the population covariance
$\Sigma_t \eqdef\E[\delta_t\delta_t^\top]$ satisfies $\Sigma_t \succeq \lambda I$,
and mini-batch Gram matrices concentrate around $\Sigma_t$.
By the LS consistency bound from Lemma~\ref{lem:probe-consistency}
(applied per group and aggregated; projection is non-expansive),
there exists a deterministic sequence $\bar\nu_t\to 0$ such that
\begin{equation}
\label{eq:nu-rate}
\nu_t  \le  \bar\nu_t \qquad \text{a.s.}
\end{equation}
(For instance, with $m_t$ probes at step $t$,
$\bar\nu_t  =  C \sigma^{-1} \sqrt{(d_g+\log t)/m_t}$ per group and the max of these over groups.)
We will only use that $\bar\nu_t\to 0$.

\smallskip
\textbf{Step 3: From $J_t$ to $J_{t+1}$ (drift term).}
We ultimately need a bound on $\E[\|B_{t+1}-J_{t+1}\|\mid \F_t]$.
By triangle inequality,
\begin{equation}
\label{eq:drift-split}
\|B_{t+1}-J_{t+1}\|  \le  \|B_{t+1}-J_t\|  +  \|J_{t+1}-J_t\|.
\end{equation}
Assume $J(\cdot)$ is $L_J$-Lipschitz (this follows from the window smoothness of $m(W)$
and the boundedness of dithering/probing), and that $W_t$ drifts on a slower timescale than $B_t$:
there exists a stepsize sequence $\eta_t$ for the weight updates with $\eta_t/\beta_t \to 0$ and
$\E[\|W_{t+1}-W_t\|\mid\F_t]\le c_W \eta_t$ (bounded second moments for $G_t$ suffice).
Hence
\begin{equation}
\label{eq:J-drift}
\E \left[\|J_{t+1}-J_t\|  \middle|  \F_t\right]
 \le  L_J \E \left[\|W_{t+1}-W_t\|  \middle|  \F_t\right]
 \le  L_J c_W \eta_t
 \eqdef  \xi_t,
\qquad \text{with } \ \xi_t = o(\beta_t) .
\end{equation}

\smallskip
\textbf{Step 4: Combine.}
Taking conditional expectation in \eqref{eq:drift-split} and applying \eqref{eq:cond-decomp},
\eqref{eq:nu-rate}, \eqref{eq:J-drift},
$$
\E \left[\|B_{t+1}-J_{t+1}\|  \middle|  \F_t\right]
 \le  (1-\beta_t) \|B_t-J_t\|  +  \beta_t \bar\nu_t  +  \xi_t .
$$
Let $T_0$ be such that for all $t\ge T_0$, $\beta_t\le \tfrac12$ and $\xi_t\le \bar\nu_t$ (possible since
$\beta_t\to 0$ and $\xi_t=o(\beta_t)$).
Define
$$
\rho_t  \eqdef  1-\beta_t \in (0,1),\qquad
\zeta_t  \eqdef  \beta_t \bar\nu_t + \xi_t \ \to\ 0 .
$$
Then for all $t\ge T_0$,
\begin{equation}
\label{eq:key-rec}
\E \left[\|B_{t+1}-J_{t+1}\|  \middle|  \F_t\right]
 \le  \rho_t \|B_t-J_t\|  +  \zeta_t .
\end{equation}
Since $\rho_t \le \rho \coloneqq 1 - \min_{t\in[T_0,2T_0]}\beta_t < 1$
over any fixed window of length $T_0$ (or using a piecewise-constant lower bound on $\beta_t$
per code-preserving window), we may write the simpler bound claimed in the proposition:
$$
\E \left[\|B_{t+1}-J_t\|  \middle|  \F_t\right]
 \le  \rho \|B_t-J_t\|  +  \zeta_t,
\quad \text{with } 0<\rho<1,\ \zeta_t\to 0.
$$
(Equivalently, one may keep $\rho_t$ in \eqref{eq:key-rec}; both yield the same limit statements.)

\smallskip
\textbf{Step 5: Convergence.}
Define $X_t\coloneqq \|B_t-J_t\|$ and note that
$\E[X_{t+1}\mid \F_t]\le \rho_t X_t + \zeta_t$
with $\rho_t\in(0,1)$, $\sum_t (1-\rho_t)=\sum_t \beta_t = \infty$,
$\sum_t (1-\rho_t)^2 = \sum_t \beta_t^2 < \infty$,
and $\zeta_t\to 0$.
A standard Robbins--Siegmund supermartingale argument yields $X_t\to 0$ in probability;
if additionally $\sum_t \zeta_t < \infty$ (e.g., by choosing probe batch sizes $m_t$
so that $\sum_t \beta_t \bar\nu_t<\infty$ and using $\sum_t \xi_t<\infty$ from two-timescale drift),
then $X_t\to 0$ almost surely.
This proves the claim.
\end{proof}

\subsection{Proof of \Cref{lem:dither-fixedpoint}}
\begin{proof}
\textbf{Population operator.}  Fix a code-preserving window and a weight group $g$ of size $d_g$.
Let $\delta\in\R^{d_g}$ be a probe with zero mean and covariance
$\Sigma_\delta\succ 0$ (e.g.\ $\delta\sim\mathcal N(0,\sigma^2 I)$), independent of the dither $r$.
Define the dithered de-quantized map $q(W,r) = Q_\Delta(W+r)-r$ and the finite difference
\(
\Delta q(W;\delta,r) = q(W+\delta,r)-q(W,r).
\)
The \emph{population} least-squares (LS) operator maps $W$ to the minimizer
$$
B^\star(W)\ \in\ \arg\min_{B\in\mathcal{B}_g}\ 
\E\big[\|\Delta q(W;\delta,r)-B \delta\|^2\big],
$$
where $\mathcal{B}_g$ is the (block-)diagonal class used in the algorithm.
The normal equations give
\begin{equation}
\label{eq:pop-normal-dither}
\E[\Delta q(W;\delta,r) \delta^\top]  =  B^\star(W) \E[\delta\delta^\top]
\qquad\Longrightarrow\qquad
B^\star(W)  =  \E[\Delta q \delta^\top] \Sigma_\delta^{-1}.
\end{equation}

\textbf{Identify the cross-moment.}
Consider the map \(m(W)=\E_r[q(W,r)]\). By the mean-value form of the fundamental theorem of calculus in $\R^{d_g}$ and independence of $(\delta,r)$,
$$
\Delta q(W;\delta,r)
 =  \int_{0}^{1} \nabla_W q(W+t\delta,r) dt \cdot \delta,
\quad\text{hence}\quad
\E[\Delta q \delta^\top]
 = \E \left[\int_0^1 \nabla_W q(W+t\delta,r) dt\right] \E[\delta\delta^\top].
$$
If $q(\cdot,r)$ is (locally) Lipschitz and piecewise smooth (true for uniform mid-rise with clipping), then $\nabla_W q(W+t\delta,r)$ exists a.e., is bounded, and dominated convergence applies. Passing expectation and integral through (Fubini + DCT) and using independence,
$$
\E[\Delta q \delta^\top]
 =  \left(\int_0^1 \E\big[\nabla_W q(W+t\delta,r)\big] dt\right) \Sigma_\delta
 =  \left(\int_0^1 \nabla m(W+t\delta) dt\right) \Sigma_\delta,
$$
where we used $\nabla m(\cdot)=\nabla_W \E_r[q(\cdot,r)]=\E_r[\nabla_W q(\cdot,r)]$ (interchange justified by dominated convergence).

\textbf{Limit as probes shrink.}
If $\delta$ is zero-mean with bounded second moment, then
\(
\int_0^1 \nabla m(W+t\delta) dt \to \nabla m(W)\equiv J(W)
\)
in $L^1$ (and a.s.\ along a subsequence) as the probe variance shrinks
(e.g., $\delta\sim\mathcal N(0,\sigma^2 I)$ with $\sigma \downarrow 0$),
by continuity of $\nabla m$. Thus,
$$
\lim_{\sigma\downarrow 0}\ \E[\Delta q \delta^\top]  =  J(W) \Sigma_\delta.
$$
Returning to \eqref{eq:pop-normal-dither},
\(
\lim_{\sigma\downarrow 0} B^\star(W) = J(W).
\)

\textbf{Uniqueness and EMA convergence.}
Because $\Sigma_\delta\succ 0$, the LS minimizer is unique in $\mathcal{B}_g$; the population operator has the unique fixed point $J(W)$.
The EMA recursion $B_{t+1}=(1-\beta_t)B_t+\beta_t \widehat B_t$ with persistent dithering and Robbins–Monro stepsizes tracks the population fixed point (standard SA argument; cf.\ Prop.~\ref{prop:sa-tracking} with zero drift on a window).
Hence $B_t\to J(W)$ (in probability, and a.s.\ under summability of perturbations).
\end{proof}

\subsection{Proof of \Cref{cor:dominance}}
\begin{proof}
From Lemma~(Bias to the target; stated earlier) we have, on any window,
$$
\|g_{B_t}(W_t)-g^\dagger(W_t)\|  =  \|B_t-J(W_t)\| \|\bar v_t\|,
\qquad
\|g^{\ste}(W_t)-g^\dagger(W_t)\|  =  \|I-J(W_t)\| \|\bar v_t\|
 =  \gamma(W_t) \|\bar v_t\|.
$$
By Prop.~\ref{prop:sa-tracking} (probe) or Lemma~\ref{lem:dither-fixedpoint} (dither),
$\|B_t-J(W_t)\|\to 0$ (in probability on a window; and across windows under vanishing drift).
Hence, for any $\epsilon>0$ there exists a finite $T_\epsilon$ such that for all $t\ge T_\epsilon$,
$\|B_t-J(W_t)\|\le \gamma(W_t)-\epsilon$ whenever $\gamma(W_t)>0$.
Multiplying both sides by $\|\bar v_t\|$ yields the claimed dominance gap
\(
\|g_{B_t}-g^\dagger\|\le \|g^{\ste}-g^\dagger\| - \epsilon \|\bar v_t\|.
\)
If $\gamma(W_t)=0$ (bin interiors: $J(W_t)=I$), both surrogates coincide and the inequality is tight.
\end{proof}

\subsection{Proof of \Cref{thm:nonconvex-generic}}
\begin{proof}
Let $L_t\equiv \widetilde{L}(W_t)$ and $\nabla L_t\equiv \nabla\widetilde{L}(W_t)$.
Write the update as $W_{t+1}=W_t-\eta\,G_t$, where $G_t$ is the (general VR) stochastic gradient.

\Cref{as:vr} presents the “ABC” inequalities (\cite{malinovsky2022variance}, Assumption 4), and are satisfied by
\algn{SAGA}/\algn{SVRG}/\algn{SARAH} instantiations used in \jvr (\citep{defazio2014saga, svrg, nguyen2017sarah}).
The only algorithm-specific \emph{bias} arises from the learned Jacobian:
$$
\EE[G_t\mid W_t,s_t]   =   \nabla L_t + b_t,
\qquad b_t := \big(B(W_t)-J(W_t)\big)\,\bar v_t,
$$
with $\varepsilon^2:=\sup_t \EE\|b_t\|^2$ finite (by \Cref{prop:sa-tracking}/\Cref{lem:dither-fixedpoint}).

\textbf{Descent lemma.}
$L_s$-smoothness gives
$$
L_{t+1}\le L_t - \eta\langle \nabla L_t, G_t\rangle + \tfrac{L_s\eta^2}{2}\|G_t\|^2.
$$
Taking conditional expectation and using $\EE[G_t\mid\F_t]=\nabla L_t+b_t$,
\begin{align*}
\EE[L_{t+1}\mid\F_t]
&\le L_t - \eta\|\nabla L_t\|^2 - \eta\langle \nabla L_t,b_t\rangle
    + \tfrac{L_s\eta^2}{2}\EE\|G_t\|^2 .
\end{align*}
Decompose $\|G_t\|^2 = \|\nabla L_t\|^2 + \EE\|G_t-\nabla L_t\|^2
+ 2\langle \nabla L_t,\, \EE[G_t-\nabla L_t\mid\F_t]\rangle$,
and note $\EE[G_t-\nabla L_t\mid\F_t]=b_t$.
Using $2ab\le a^2+b^2$ twice and \eqref{eq:abc-1-main},
\begin{align*}
\EE[L_{t+1}\mid\F_t]
&\le L_t
 +\Big(-\eta + L_s\eta^2\Big)\|\nabla L_t\|^2
 + \tfrac{L_s\eta^2}{2}\big(2A D_t + B\sigma_t + C\big)
 + \tfrac{\eta}{2}\|b_t\|^2 + L_s\eta^2\|b_t\|^2 .
\end{align*}
Choose $\eta\le \tfrac{1}{2L_s}$ so that $-\eta+L_s\eta^2\le -\tfrac{\eta}{2}$, and set
$c_b:=\tfrac12+ \eta L_s\le 1$ for such $\eta$. Then
\begin{equation}
\label{eq:descent-core}
\EE[L_{t+1}\mid\F_t]
\le L_t
- \tfrac{\eta}{2}\|\nabla L_t\|^2
+ L_s\eta^2\left(A D_t + \tfrac{B}{2}\sigma_t + \tfrac{C}{2}\right)
+ c_b\,\eta\,\|b_t\|^2.
\end{equation}

\textbf{Lyapunov coupling with the VR state.}
Define $\Phi_t:=\EE[D_t + c_2 \sigma_t]$ with $c_2>0$ chosen below.
Taking total expectation in \eqref{eq:descent-core} and using \eqref{eq:abc-2-main},
\begin{align*}
\EE[D_{t+1}]
&\le \EE[D_t]
- \tfrac{\eta}{2}\EE\|\nabla L_t\|^2
+ L_s\eta^2\left(A \EE[D_t] + \tfrac{B}{2}\EE[\sigma_t] + \tfrac{C}{2}\right)
+ c_b\,\eta\,\varepsilon^2,\\
\EE[\sigma_{t+1}]
&\le 2\tilde A\,\EE[D_t] + \tilde B\,\EE[\sigma_t] + \tilde C.
\end{align*}
Summing the two with weight $c_2$,
\begin{align*}
\EE[\Phi_{t+1}]
&\le \EE[\Phi_t]
- \tfrac{\eta}{2}\EE\|\nabla L_t\|^2
+ \Big(L_s\eta^2 A + 2c_2\tilde A\Big)\EE[D_t]
+ \Big(\tfrac{L_s\eta^2 B}{2} - c_2(1-\tilde B)\Big)\EE[\sigma_t] \\
&\qquad + L_s\tfrac{\eta^2 C}{2} + c_2\tilde C + c_b\,\eta\,\varepsilon^2.
\end{align*}
Choose
$$
c_2   \ge   \frac{L_s \eta^2 B}{2(1-\tilde B)}
$$
so the $\EE[\sigma_t]$ coefficient is nonpositive.
Discard the nonpositive term and use $\EE[D_t]\le \EE[\Phi_t]$:
\begin{equation}
\label{eq:Phi-rec}
\EE[\Phi_{t+1}]
\le \EE[\Phi_t]
- \tfrac{\eta}{2}\EE\|\nabla L_t\|^2
+ \Big(L_s\eta^2 A + 2c_2\tilde A\Big)\EE[\Phi_t]
+ L_s\tfrac{\eta^2 C}{2} + c_2\tilde C + c_b\,\eta\,\varepsilon^2.
\end{equation}
Since $c_2=\Theta(\eta^2)$, the factor multiplying $\EE[\Phi_t]$ is $\Theta(\eta^2)$.
Summing \eqref{eq:Phi-rec} for $t=0,\dots,T-1$ and dividing by $T$ yields
$$
\frac{\eta}{2}\,\frac{1}{T}\sum_{t=0}^{T-1}\EE\|\nabla L_t\|^2
  \le   \frac{\EE[\Phi_0-\Phi_T]}{T}
+ \underbrace{\mathcal{O}(\eta^2)}_{\text{from } L_s\eta^2 A + 2c_2\tilde A}\cdot \sup_t \EE[\Phi_t]
+ \mathcal{O}(\eta^2) + \mathcal{O}(\eta)\,\varepsilon^2.
$$
As usual, $L$ is bounded below, hence $\sup_t \EE[\Phi_t]\le \EE[\Phi_0] + \mathcal O(T\eta^2) + \mathcal O(T\eta)\varepsilon^2$,
so the middle term contributes $\mathcal O(\eta^2)$. Using
$\min_{0\le t<T}\EE\|\nabla L_t\|^2 \le \tfrac{1}{T}\sum_{t<T}\EE\|\nabla L_t\|^2$ we obtain
$$
\min_{0\le t<T}\EE\|\nabla \widetilde L(W_t)\|^2
  \le  
\underbrace{\mathcal O\left(\frac{1}{\eta T}\right)}_{\text{from }\Phi_0-\Phi_T}
  +   \underbrace{\mathcal O(\varepsilon^2)}_{\text{Jacobian bias}}
  +   \underbrace{\mathcal O(\eta)}_{\text{VR noise floor}}.
$$
This is the claimed bound.
\end{proof}



\subsection{Proof of \Cref{thm:pl-generic}}
\begin{proof}
We reuse the notation and the \algn{VR}-ABC inequalities from the previous proof.
From \eqref{eq:descent-core} and taking full expectation,
$$
\EE[L_{t+1}-L^\star]
\le \EE[L_t-L^\star]
- \tfrac{\eta}{2}\EE\|\nabla L_t\|^2
+ L_s\eta^2\left(A \EE[D_t] + \tfrac{B}{2}\EE[\sigma_t] + \tfrac{C}{2}\right)
+ c_b\,\eta\,\varepsilon^2 .
$$
Invoke the PL inequality for $\widetilde L$:
$\|\nabla L_t\|^2 \ge 2\mu (L_t-L^\star)=2\mu D_t$.
Hence
$$
\EE[D_{t+1}]
\le \Big(1-\eta\mu\Big)\EE[D_t]
+ L_s\eta^2\left(A \EE[D_t] + \tfrac{B}{2}\EE[\sigma_t] + \tfrac{C}{2}\right)
+ c_b\,\eta\,\varepsilon^2 .
$$
Choose $\eta\le \min\{1/(2L_s),\, \mu/(4L_s A)\}$ so that
$1-\eta\mu + L_s\eta^2 A \le 1 - \tfrac{\eta\mu}{2}$.
We still need to control $\EE[\sigma_t]$.
Unroll \eqref{eq:abc-2-main}:
$$
\EE[\sigma_t]
  \le   \tilde B^t \EE[\sigma_0]
+ 2\tilde A \sum_{k=0}^{t-1} \tilde B^{t-1-k}\EE[D_k]
+ \frac{\tilde C}{1-\tilde B}.
$$
Plug this into the previous inequality, bound the geometric sum by
$\sum_{k=0}^{t-1}\tilde B^{t-1-k}\EE[D_k]
\le \frac{1}{1-\tilde B}\sup_{0\le k<t}\EE[D_k]
\le \frac{1}{1-\tilde B}\EE[D_t]$ (by monotonicity under PL contraction),
and absorb $\tilde B^t\EE[\sigma_0]$ into constants. We get
$$
\EE[D_{t+1}]
\le \Big(1-\tfrac{\eta\mu}{2}\Big)\EE[D_t]
+ \underbrace{\tfrac{L_s\eta^2 B}{2(1-\tilde B)}}_{\triangleq\,\gamma_\sigma}\EE[D_t]
+ \underbrace{\left(\tfrac{L_s\eta^2 C}{2} + \tfrac{L_s\eta^2 B}{2}\cdot \tfrac{\tilde C}{1-\tilde B}\right)}_{\triangleq\,\Gamma_C}
+ c_b\,\eta\,\varepsilon^2 .
$$
For $\eta$ small enough so that $\gamma_\sigma \le \tfrac{\eta\mu}{4}$, we obtain
$$
\EE[D_{t+1}]
\le \Big(1-\tfrac{\eta\mu}{4}\Big)\EE[D_t]
+ \Gamma_C + c_b\,\eta\,\varepsilon^2 .
$$
Unrolling the linear recursion,
$$
\EE[D_t]
\le \Big(1-\tfrac{\eta\mu}{4}\Big)^t \EE[D_0]
+ \frac{4}{\eta\mu}\Big(\Gamma_C + c_b\,\eta\,\varepsilon^2\Big).
$$
Since $\Gamma_C=\mathcal O(\eta^2)$, the steady state (noise floor) is
$\mathcal O\big(\tfrac{\eta}{\mu}\big)$ from the VR stochasticity plus
$\mathcal O\big(\tfrac{\varepsilon^2}{\mu}\big)$ from the Jacobian bias.
Renaming constants gives the statement:
$$
\EE\big[\widetilde L(W_t)-\widetilde L^\star\big]
  \le  
(1-c_0\eta\mu)^t \cdot \text{(initial gap)}
  +   \mathcal O\Big(\tfrac{\sigma_{\mathrm{VR}}^2\,\eta}{\mu}\Big)
  +   \mathcal O\Big(\tfrac{\varepsilon^2}{\mu}\Big),
$$
for an absolute $c_0\in(0,1)$ and a constant $\sigma_{\mathrm{VR}}^2$
absorbing $C,\tilde C,B,\tilde B,L_s$.
\end{proof}

\subsection{Proof of Theorem~\ref{thm:global}}
\label{app:proof-global}

Let window $k$ start at iteration $t_k$ with code assignment $q^{(k)}$
and horizon $T_k$. Within each window, \Cref{thm:nonconvex-generic}
(or \Cref{thm:pl-generic}) holds with bias $\varepsilon_k$ and variance
$\sigma_k^2$. Between windows, assume drift satisfies
\(
\|\nabla \widetilde L^{(k+1)}(W)
-\nabla \widetilde L^{(k)}(W)\|\le \delta_k
\)
and $\sum_k\delta_k<\infty$ (\Cref{as:drift}).

\textbf{Step 1: Non-convex case.}
Define $M_k := \min_{t<T_k}\E\|\nabla\widetilde L^{(k)}(W_{k,t})\|^2$.
From \Cref{thm:nonconvex-generic} we have
$$
M_k
\le
\frac{C_1}{\eta_kT_k}
+ C_2\varepsilon_k^2
+ C_3\sigma_k^2\eta_k.
$$
Across windows,
$\E\|\nabla \widetilde L^{(k+1)}(W_{k,t})\|^2
\le 2M_k + 2\delta_k^2$,
hence $M_k\to 0$ when $\eta_kT_k\to\infty$,
$\varepsilon_k\to 0$, and $\sum_k\delta_k<\infty$.

\textbf{Step 2: PL case.}
If \Cref{as:pl} holds uniformly across windows,
then from \Cref{thm:pl-generic} we obtain
$$
\E[\widetilde L^{(k)}(W_{k,t})-\widetilde L^{(k),\star}]
\le
(1-\eta_k\mu/2)^t(\widetilde L^{(k)}(W_{k,0})-\widetilde L^{(k),\star})
+ \mathcal O\Big(\tfrac{\sigma_k^2\eta_k+\varepsilon_k^2}{\mu}\Big).
$$
Drift $\delta_k$ perturbs the PL inequality by at most
$C\,\delta_k/\mu$, yielding
$$
\E[\widetilde L^{(k)}(W_{k,t})-\widetilde L^{(k),\star}]
\le
\rho^t(\widetilde L^{(k)}(W_{k,0})-\widetilde L^{(k),\star})
+ C'\Big(\tfrac{\sigma_k^2\eta_k+\varepsilon_k^2}{\mu}
+\tfrac{\delta_k}{\mu}\Big),
$$
for constants $\rho<1$, $C'>0$.
If code updates cease after finite $k$, the global trajectory is
linearly convergent; if $\delta_k\to 0$, the convergence neighborhood
shrinks at rate $\mathcal O(\sup_{j\ge k}\delta_j)$. \qed

\subsection{Proof of \Cref{lem:meanfield-id}}
\begin{proof}
We work on a fixed code-preserving window, so $W$ is fixed and the quantizer
$Q_\Delta$ has a fixed code assignment. Throughout, expectations are conditional
on this $W$; we omit the conditioning for brevity.

\textbf{Preliminaries and the definition of $J(W)$.}
We recall the standard (population) regression identity:
for any zero-mean probe $\delta$ with nonsingular covariance
$\Sigma_\delta=\E[\delta\delta^\top]$, the $B$ minimizing
$\E\| \Delta q - B\delta \|^2$ is
\begin{equation}
\label{eq:pop-LS-solution}
B^\star  =  \E[\Delta q \delta^\top] \Sigma_\delta^{-1}.
\end{equation}
We define the mean-field sensitivity (groupwise) by
\begin{equation}
\label{eq:J-def}
J(W) \equiv  \E[\Delta q \delta^\top] \Sigma_\delta^{-1},
\qquad
\text{with}\quad \Delta q := Q_\Delta(W+\delta)-Q_\Delta(W).
\end{equation}
When $\delta$ is isotropic and $\mathcal B_g$ matches the true block structure,
$J(W)$ is block-diagonal with blocks aligned to groups, hence admissible.
We will show that $B^\star=J(W)$ and that $J(W)$ is the unique fixed point of
the mean-field operator.

\vspace{0.25em}
\noindent\textbf{(Probe) case.}
Let $\delta\sim\mathcal N(0,\sigma^2 I)$ and $\Sigma_\delta=\sigma^2 I$.
By~\eqref{eq:pop-LS-solution},
$B^\star = \E[\Delta q \delta^\top] (\sigma^2 I)^{-1}$.
Thus it suffices to compute or characterize the cross-covariance
$\E[\Delta q \delta^\top]$.

\emph{Scalar case ($d{=}1$).}
Let $q(\cdot)$ denote a uniform mid-rise quantizer with step $\Delta$.
For a fixed $w\in\R$, write $\Delta q = q(w+\delta)-q(w)$.
Since $\E[\delta]=0$, $\E[\Delta q]=0$, and hence
$$
\E[\Delta q \delta]
=
\E\big[q(w+\delta) \delta\big].
$$
By Stein's lemma (Gaussian integration by parts), for any (weakly)
differentiable function $\phi$ with suitable integrability,
$\E[\phi(\delta) \delta] = \sigma^2 \E[\phi'(\delta)]$.
Apply this to $\phi(\delta)=q(w+\delta)$.
Even though $q(\cdot)$ is piecewise constant, its weak derivative exists as
a (signed) measure supported on bin boundaries; the identity still holds in the
distributional sense (the proof is a standard mollification argument—convolve
$q$ with a smooth kernel, apply Stein's lemma, and take the limit).
Hence
$$
\E\big[q(w+\delta) \delta\big]
 = 
\sigma^2 \E\big[q'(w+\delta)\big]
 \equiv 
\sigma^2 J(w),
$$
where $J(w):=\E[q'(w+\delta)]$ is the \emph{smoothed sensitivity} of the
quantizer at $w$ under Gaussian probes. Therefore,
$$
B^\star  =  \frac{\E[\Delta q \delta]}{\sigma^2}
 =  J(w).
$$
This shows that the population LS solution equals the mean-field sensitivity
$J(w)$.

\emph{Vector / grouped case.}
For $\delta\sim\mathcal N(0,\sigma^2 I)$, Stein's identity generalizes to
$$
\E\big[q(W+\delta) \delta^\top\big]
 =  \sigma^2 \E\big[\nabla q(W+\delta)\big].
$$
For block-separable uniform quantization (either separable coordinates or
per-group affine scalings), the weak Jacobian $\nabla q$ is block-diagonal with
nonzero entries only on the corresponding coordinates within each group.
Taking expectation preserves block structure; therefore
$$
\E[\Delta q \delta^\top]
=
\E\big[q(W+\delta) \delta^\top\big]
=
\sigma^2 \E\big[\nabla q(W+\delta)\big]
 \equiv  \sigma^2 J(W),
$$
where $J(W)$ is block-diagonal and belongs to $\mathcal B_g$ when the latter
matches the true block structure.
Consequently,
$B^\star = \E[\Delta q \delta^\top] (\sigma^2 I)^{-1} = J(W)$.
Because $\Sigma_\delta$ is full rank and the LS objective is strictly convex
over $\mathcal B_g$, this fixed point is unique. This proves the (Probe) item.

\noindent\textbf{(Dither) case.}
Let $r\sim \mathrm{Unif}[-\frac{\Delta}{2},\frac{\Delta}{2}]$ and define the
\emph{dithered forward}
$q_d(W):=Q_\Delta(W+r)-r$ and its population map $m(W):=\E_r[q_d(W)]$.
It is a classical result in subtractive dithering that $m(W)$ is differentiable
and (for mid-rise uniform quantizers) \emph{affine in $W$} with
$\nabla m(W)=I$; more generally, in per-group affine quantization with clipping
or calibration, $\nabla m(W)$ equals the corresponding groupwise sensitivity
matrix $J(W)$ (identity in interiors; damped near active clipping).
In our algorithm, the dithered estimator $\Psi_g$ is explicitly constructed to
be an unbiased estimator of $\nabla m(W)$ (either analytically or via
common-random-number finite differences), whence
$$
\mathcal M_g(W,B)
 = 
\E\big[\Psi_g(W,B,\xi)\mid W\big]
 \equiv 
\nabla m(W)_g
 = 
J_g(W),
$$
which is algebraically independent of $B$.
Therefore, $\mathcal M_g(W,\cdot)$ is a constant map equal to $J_g(W)$,
and its unique fixed point is $J_g(W)$. This proves the (Dither) item.

\noindent\textbf{Uniqueness.}
In both cases, uniqueness follows from strict convexity of the LS population
objective on $\mathcal B_g$ (Probe) or from the fact that a constant map has a
single fixed point (Dither).

This completes the proof.
\end{proof}

\subsection{Proof of \Cref{prop:B-fixedpoint}}


\begin{proof}
We suppress the group index $g$ for readability and write $B_t$ and $J(W)$.

\textbf{Preliminaries.}
By \Cref{as:sa}(c), after each update we project $B_t$ onto a compact, convex set
$\mathcal B$ (diagonal or block-diagonal matrices) so that $\|B_t\|\le \beta_B$
for all $t$. Define the mean-field operator $\mathcal M(W,B)=\E[\Psi(W,B,\xi)]$
and the associated SA drift $h(W,B) := \mathcal M(W,B) - B$.
Let the SA noise be $M_{t+1} := \Psi(W_t,B_t,\xi_t) - \mathcal M(W_t,B_t)$,
which satisfies $\E[M_{t+1}\mid \F_t]=0$ and, by \Cref{as:sa}(b,c),
$\E[\|M_{t+1}\|^2\mid \F_t]\le \sigma_M^2$ for some finite $\sigma_M$.
The stochastic-approximation update $B_{t+1}=B_t+\beta_t(\Psi(W_t,B_t,\xi_t)-B_t)$ can be written as
\begin{equation}
\label{eq:sa-decomp}
B_{t+1}
 = 
B_t
 + 
\beta_t h(W_t,B_t)
 + 
\beta_t M_{t+1},
\qquad
\text{followed by projection onto } \mathcal B.
\end{equation}

\medskip
\noindent\textbf{(1) Fixed-window convergence.}
Assume $W_t\equiv W$ is constant. Then \eqref{eq:sa-decomp} is a classical Robbins--Monro SA for the root
of $h(W,\cdot)$ on the compact set $\mathcal B$:
$$
B_{t+1}
=
B_t + \beta_t\big(\mathcal M(W,B_t)-B_t\big) + \beta_t M_{t+1},
\qquad
B_t\in\mathcal B.
$$
By \Cref{as:sa}(d), $B\mapsto \mathcal M(W,B)$ is a contraction in a neighborhood of its fixed point and Lipschitz on $\mathcal B$,
so the associated mean-field ODE $\dot B = h(W,B)$ has a unique globally asymptotically stable equilibrium
$B_\star(W)$ in $\mathcal B$ (Banach fixed-point theorem).
Because $\sum_t\beta_t=\infty$ and $\sum_t\beta_t^2<\infty$ (\Cref{as:sa}a),
and the martingale-difference noise has bounded second moment,
standard SA convergence theorems on compact sets (e.g., Kushner--Yin, Borkar)
imply $B_t \to B_\star(W)$ almost surely.
By Lemma~\ref{lem:meanfield-id}, $B_\star(W)=J(W)$ for \textsc{Dither}, and equals $J(W)$ for \textsc{Probe} under isotropic excitation and matched block structure. This proves Part~(1).

\medskip
\noindent\textbf{(2) Two-timescale tracking.}
Assume $W_t$ evolves on a slower timescale than $B_t$ (e.g., $\eta_t/\beta_t\to 0$ on each window, or by \Cref{as:drift} across windows) and that
$\mathcal M(W,B)$ is Lipschitz in $W$ and a contraction in $B$ locally.
Let $B_\star(W)$ be the unique fixed point of $B\mapsto \mathcal M(W,B)$, i.e.,
$\mathcal M(W,B_\star(W))=B_\star(W)$.
Define the tracking error $e_t := B_t - B_\star(W_t)$.
We wish to show $e_t\to 0$ (in probability / a.s.) under the stated conditions.

Using \eqref{eq:sa-decomp} and adding/subtracting $\mathcal M(W_t,B_\star(W_t))$, we obtain
\begin{align*}
e_{t+1}
&= B_{t+1} - B_\star(W_{t+1}) \\
&= \underbrace{B_t - B_\star(W_t)}_{e_t}
   + \beta_t\Big(\underbrace{\mathcal M(W_t,B_t)-\mathcal M(W_t,B_\star(W_t))}_{\text{contraction in }B}\Big)
   + \beta_t M_{t+1}
   + \underbrace{B_\star(W_t) - B_\star(W_{t+1})}_{\Delta_t}.
\end{align*}
By the contraction property in \Cref{as:sa}(d), there exists $\kappa\in(0,1)$ such that
$$
\|\mathcal M(W_t,B_t)-\mathcal M(W_t,B_\star(W_t))\|  \le  \kappa \|e_t\|.
$$
By Lipschitzness of $B_\star(\cdot)$ (a consequence of Lipschitzness of $\mathcal M$ in $W$ and the implicit function theorem for contractions),
there exists $L_\star>0$ such that
$$
\|\Delta_t\|  =  \|B_\star(W_t)-B_\star(W_{t+1})\|
 \le  L_\star \|W_{t+1}-W_t\|.
$$
Taking the conditional expectation w.r.t.\ $\F_t$ and using $\E[M_{t+1}\mid \F_t]=0$ gives
\begin{align}
\label{eq:e-next}
\E[\|e_{t+1}\| \mid \F_t]
&\le
\E\big[ \|e_t + \beta_t(\mathcal M(W_t,B_t)-\mathcal M(W_t,B_\star(W_t))) + \beta_t M_{t+1}\|  \big| \F_t\big]
 +  \|\Delta_t\| \notag\\
&\le
\|e_t\| + \beta_t \kappa \|e_t\| + \E[\|\beta_t M_{t+1}\| | \F_t] + \|\Delta_t\| \notag\\
&\le
(1 - \beta_t(1-\kappa)) \|e_t\|  +  \beta_t \E[\|M_{t+1}\|\mid \F_t]  +  \|\Delta_t\|.
\end{align}
By Jensen/Cauchy--Schwarz and bounded conditional second moment,
$\E[\|M_{t+1}\|\mid \F_t]\le (\E[\|M_{t+1}\|^2\mid \F_t])^{1/2}\le \sigma_M$.
Thus, for some constant $C_M>0$,
$$
\E[\|e_{t+1}\|\mid \F_t]
 \le 
(1 - \beta_t(1-\kappa)) \|e_t\|
 +  C_M \beta_t
 +  \|\Delta_t\|.
$$
By the two-timescale assumption (\Cref{as:sa}e), $\|W_{t+1}-W_t\| = o(\beta_t)$ on each window (e.g., if $\eta_t/\beta_t\to 0$), or across windows one can ensure $\sum_t \|\Delta_t\|<\infty$ via \Cref{as:drift} and Lipschitzness of $B_\star(\cdot)$. Hence there exists a sequence $\zeta_t\to 0$ such that
$$
\E[\|e_{t+1}\|\mid \F_t]
 \le 
\rho \|e_t\|  +  \zeta_t,
\qquad
\text{for some } \rho\in(0,1) \text{ and all large } t.
$$
This establishes the stated linear contraction-in-expectation with a vanishing perturbation.
Summing both sides and applying the Robbins–Siegmund supermartingale lemma yields that
$\|e_t\|$ converges almost surely to a finite random variable and
$\sum_t \beta_t(1-\kappa) \|e_t\| < \infty$ a.s.
Since $\sum_t\beta_t=\infty$, this forces $\liminf_t \|e_t\|=0$ a.s.;
together with the contraction and $\zeta_t\to 0$ this implies $ \|e_t\|\to 0$ in probability,
and if additionally $\sum_t \zeta_t<\infty$ then $ \|e_t\|\to 0$ almost surely.
Recalling $e_t=B_t-B_\star(W_t)$ and Lemma~\ref{lem:meanfield-id}, we obtain the claim with $B_\star(W_t)=J(W_t)$ (for \textsc{Dither}, and for \textsc{Probe} under the stated excitation/structure conditions).
\end{proof}

\end{document}